\documentclass{article}

\usepackage[accepted]{icml2024}

\usepackage[utf8]{inputenc} 
\usepackage[T1]{fontenc}    
\usepackage{url}            
\usepackage{booktabs}       
\usepackage{amsfonts}       
\usepackage{nicefrac}       
\usepackage{microtype}      
\usepackage{wrapfig}
\usepackage{subcaption}
\usepackage[algo2e]{algorithm2e}

\usepackage{amssymb,amsmath}
\usepackage{amsthm}
\usepackage{mathrsfs}
\usepackage{thmtools, thm-restate}
\usepackage[mathscr]{eucal} 
\usepackage{natbib}
\usepackage{bbm}
\usepackage{arydshln}
\usepackage{enumitem}
\usepackage{cancel}
\usepackage{graphicx}
\usepackage{mdframed}
\usepackage{pgfplots} 
\usepackage{changepage}

\usepackage[tickmarkheight=.5em,textwidth=\marginparwidth,textsize=small, color=green!40]{todonotes}
\theoremstyle{plain}
\newtheorem{theorem}{Theorem}[section]

\newtheorem{lemma}[theorem]{Lemma}

\theoremstyle{definition}

\theoremstyle{remark}

\pgfplotsset{compat=1.18}
\definecolor{mydarkblue}{rgb}{0,0.08,0.45}
\definecolor{mydarkgreen}{rgb}{0,0.45,0.08}

\usepackage[colorlinks=true,
    linkcolor=mydarkblue,
    citecolor=mydarkblue,
    filecolor=mydarkblue,
    urlcolor=mydarkblue,
    pagebackref=true]{hyperref} 
\usepackage[capitalize]{cleveref}
\renewcommand*\backref[1]{\ifx#1\relax \else (cited on #1) \fi}

\mdfdefinestyle{mdframedthmbox}{%
	leftmargin=.0\textwidth,
	rightmargin=.0\textwidth,%
	innertopmargin=0.75em,
	innerleftmargin=.5em,
	innerrightmargin=.5em,
}

\usepackage{float}

\crefname{lemmanew}{Lemma}{Lemmas}

\newcommand{\inner}[2]{\langle #1, #2 \rangle}

\newcommand{\E}{\mathbb{E}}

\newcommand{\margin}{\chi}

\def \nus/{\texttt{NUS}}

\newcommand{\nleft}{\mathclose\bgroup\left}
\newcommand{\nright}{\aftergroup\egroup\right}

\def\1{\bm{1}}

\DeclareMathAlphabet{\mathsfit}{\encodingdefault}{\sfdefault}{m}{sl}
\SetMathAlphabet{\mathsfit}{bold}{\encodingdefault}{\sfdefault}{bx}{n}

\def\cC{{\mathcal{C}}}
\def\cD{{\mathcal{D}}}

\def\cF{{\mathcal{F}}}

\newcommand{\R}{\mathbb{R}}

\newcommand{\norm}[1]{\left\|#1\right\|}
\newcommand{\normsq}[1]{\left\|#1\right\|^{2}}

\newcommand{\tightsub}[1]{{\kern -.1em \raise-.1em\hbox{\tiny$#1$}}{}}

\icmltitlerunning{{From Inverse Optimization to Feasibility to ERM}}

\begin{document}

\twocolumn[
\icmltitle{{From Inverse Optimization to Feasibility to ERM}}



\icmlsetsymbol{equal}{*}

\begin{icmlauthorlist}
\icmlauthor{Saurabh Mishra}{sfu}
\icmlauthor{Anant Raj}{inria}
\icmlauthor{Sharan Vaswani}{sfu}

\end{icmlauthorlist}
\icmlaffiliation{sfu}{Simon Fraser University}
\icmlaffiliation{inria}{SIERRA Project Team (Inria), Coordinated Science Laboratory (CSL), UIUC}

\icmlcorrespondingauthor{Saurabh Mishra}{skm24@sfu.ca}
\icmlcorrespondingauthor{Sharan Vaswani}{vaswani.sharan@gmail.com}

\icmlkeywords{Inverse Optimization, Predict and Optimize, Convex Feasibility, Empirical Risk Minimization}
\vskip 0.3in ]


\printAffiliationsAndNotice{} 
\begin{abstract}
Inverse optimization involves inferring unknown parameters of an optimization problem from known solutions and is widely used in fields such as transportation, power systems, and healthcare. We study the \emph{contextual inverse optimization} setting that utilizes additional contextual information to better predict the unknown problem parameters. We focus on contextual inverse linear programming (\texttt{CILP}), addressing the challenges posed by the non-differentiable nature of LPs. For a linear prediction model, we reduce \texttt{CILP} to a convex feasibility problem, allowing the use of standard algorithms such as alternating projections. The resulting algorithm for \texttt{CILP} is equipped with a linear convergence guarantee without additional assumptions such as degeneracy or interpolation. Next, we reduce \texttt{CILP} to empirical risk minimization (ERM) on a smooth, convex loss that satisfies the Polyak-Lojasiewicz condition. This reduction enables the use of scalable first-order optimization methods to solve large non-convex problems while maintaining theoretical guarantees in the convex setting. Subsequently, we use the reduction to ERM to quantify the generalization performance of the proposed algorithm on previously unseen instances. Finally, we experimentally validate our approach on synthetic and real-world problems, and demonstrate improved performance compared to existing methods.

\end{abstract}

\vspace{-4ex}
\section{Introduction}
\label{sec:introduction}

Inverse optimization~\citep{heuberger2004inverse} is the reverse of standard optimization and uses a known output (decision) of an optimization problem to infer the unknown problem parameters. For example, in the context of linear programming (LP), inverse optimization uses the optimal solutions to the LP in order to learn the coefficients (costs) that can produce these solutions. Inverse optimization has found applications in transportation~\citep{bertsimas2015data}, power systems~\citep{birge2017inverse} and healthcare~\citep{chan2022inverse} (refer to~\citet{chan2023inverse} for a recent survey). 

We focus on integrating additional contextual information into the inverse optimization framework. In particular, we leverage historical data and a machine learning (ML) model to predict the (unknown) optimization problem parameters that can render (known) optimal decisions. This setting is commonly referred to as contextual inverse optimization (CIO)~\citep{besbes2023contextual,sun2023maximum} or data-driven inverse optimization~\citep{mohajerin2018data}. CIO requires a combination of prediction and optimization and has found applications in optimal transport and vehicle routing~\citep{li2022overview}, financial modeling \citep{cornuejols2006optimization}, power systems~\citep{bansal2005optimization,li2018optimal}, healthcare~\citep{angalakudati2014business}, circuit design~\citep{boyd2001optimal}, robotics~\cite{raja2012optimal}. Some use-cases for CIO are as follows.

\begin{figure}[t]
\begin{center}
\centerline{\includegraphics[width=\columnwidth]{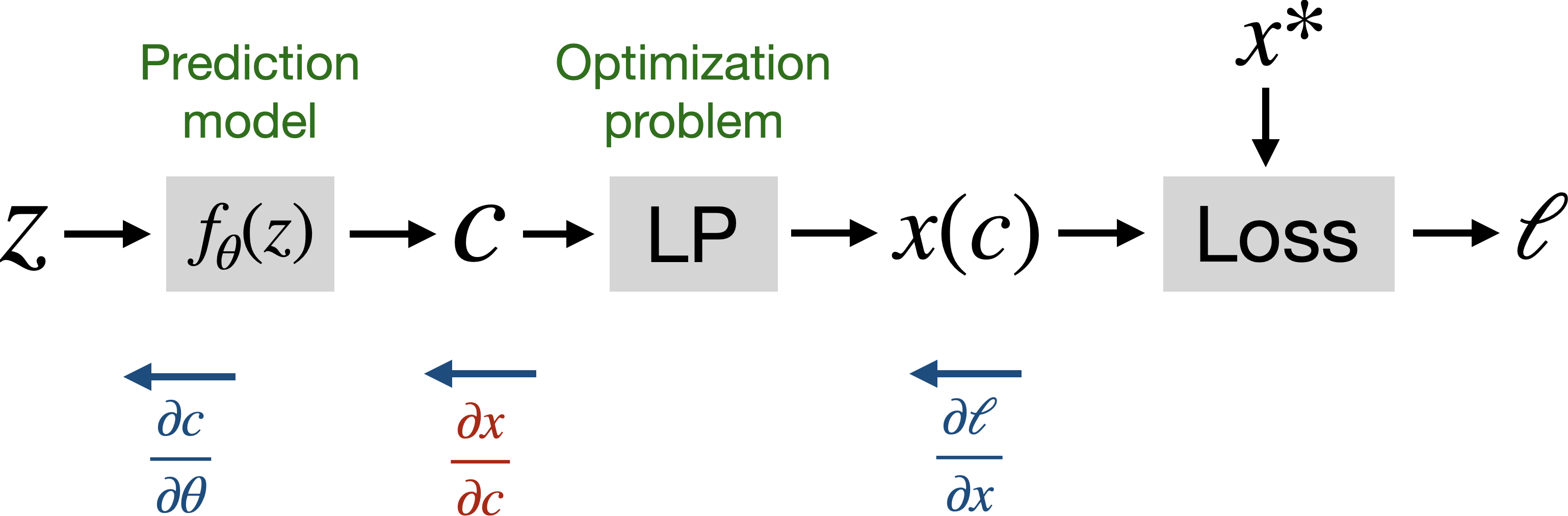}}
\caption{CIO framework: model $f_\theta$ takes input $z$ and predicts the cost vector $c = f_\theta(z)$. This cost vector is the input of an optimization procedure that outputs decision $x(c)$. Given the optimal decision $x^*$, the objective is to learn the model parameters such that the predicted decision $x(c)$ is close to the optimal decision. To train the model in an end-to-end fashion, the key challenge is to compute the gradient of $c$ w.r.t decision $x(c)$ (shown in red in the figure). }
\label{fig:intro}
\vspace{-1ex}
\end{center}
\end{figure}

\emph{Example 1}: Energy-cost aware scheduling~\citep{wahdany2023more}, which involves using weather data to forecast wind-energy generation and hence energy prices (prediction). These predictions can be used to schedule jobs (optimization) to minimize energy costs. For the CIO, the contextual information corresponds to weather data, and the solutions (decisions) correspond to past schedules.  

\emph{Example 2}: Shortest path planning~\citep{guyomarch2017warcraft}, which involves predicting the time taken through different routes or terrain (prediction). These predictions can be used to determine the shortest path between two locations (optimization). For CIO, the contextual information corresponds to images or features of the terrain, and the decisions correspond to known shortest paths for pairs of locations. 

\emph{Example 3}: Inverse reinforcement learning (IRL)~\citep{ng2000algorithms}, which involves learning the underlying reward function in a Markov decision process (MDP) from the observed behaviour of a human expert. The learned reward function can be used to infer a good policy for an artificial agent. Assuming that the human expert acts in order to maximize the implicit reward functions, for the CIO, the context corresponds to features of the MDP, and decisions correspond to the observed expert behaviour. 

\emph{Example 4}: In rational choice theory, a common way to model agents (e.g. users interacting with a recommendation system) is to assume that the (i) agent is rational and is making decisions to optimize some unknown implicit utility function and that (ii) the form (but not the parameters) of this utility function is known (e.g. whether it is linear or concave). For recommendation systems, the user's demographics and other metadata correspond to the context. In CIO, this context is used to predict the unknown parameters of the utility function, such that when it is maximized, it can explain the users' past purchases (corresponding to the known decisions)~\citep{wilder2019melding}. The estimated utility function can then be used to make better recommendations.


Since numerous combinatorial problems, including shortest path, max-flow, and perfect matching, can be cast as linear programs, we will mainly focus on cases where the optimization problem is a Linear Program (LP) (in~\cref{sec:extension-nonlinear-constraints}, we briefly consider more general non-linear problems). The examples of CIO presented earlier fall within the LP framework. For LPs, the key challenge of contextual inverse linear programming (\texttt{CILP}) lies in the non-differentiable nature of LPs. This limitation precludes the direct use of auto-differentiation techniques. To overcome this problem, we make the following contributions. 

\textbf{Contribution 1}: For a linear prediction model, we reduce \texttt{CILP} to a convex feasibility problem (\cref{sec:feasibility}). This reduction enables the use of standard algorithms such as alternating projections. Unlike existing work~\citep{sun2023maximum}, the resulting method (\cref{alg:revgrad}) guarantees linear convergence to the solution without additional assumptions such as degeneracy or interpolation. 

\textbf{Contribution 2}: To efficiently handle large-scale problems, we reduce the feasibility problem (and hence \texttt{CILP}) to empirical risk minimization (ERM) on a smooth, convex loss function satisfying the Polyak-Lojasiewicz condition~\citep{polyak1964gradient} (\cref{sec:erm}). This reduction allows us to employ scalable first-order optimization methods while retaining strong theoretical guarantees. 

\textbf{Contribution 3}: In~\cref{sec:generalization}, we argue about the shortcomings of the previous measures of performance for \texttt{CILP}, and propose a new sub-optimality metric. Subsequently, we use the reduction to ERM to quantify the performance of the proposed algorithm on previously unseen instances.



\textbf{Contribution 4}: In~\cref{sec:experiments}, we validate the effectiveness of our approach with experiments on synthetic shortest path and fractional knapsack problems~\citep{sun2023maximum}, and real-world Warcraft shortest path and MNIST perfect matching tasks~\citep{vlastelica2019differentiation}. Our empirical results demonstrate that the proposed algorithm results in improved performance compared to the prior work.  

\vspace{-2ex}
\section{Related Work}
\label{sec:related-work}
In this section, we review the related works, contrasting them with our proposed approach. 
 
\textbf{Inverse Optimization}:~\citet{iyengar2005inverse}, uses the Karush–Kuhn–Tucker (KKT) optimality conditions for the LP to define the feasible set of cost vectors. Similarly,~\citet{mohajerin2018data}, uses the Wasserstein metric to find a set of robust cost vectors by formulating an inverse optimization problem. However, they do not consider the contextual setting, so no learning is required. In contrast, we use the KKT conditions to train a prediction model, mapping contextual information to optimal decisions. More recently,~\citet{besbes2023contextual} consider solving CIO in both the online and offline settings. Their offline setting is similar to our problem formulation, but does not make any linearity or convexity assumptions. They derive bounds on the worst-case suboptimality for a specific mapping from features to cost vectors. However, it is unclear whether this mapping can be efficiently computed even in the special case of LPs. Furthermore,~\citet{besbes2023contextual} assume realizability i.e. there is no noise in the decisions and the model can perfectly fit (interpolate) the data. This further limits the practical utility of their framework. In contrast, we make no realizability assumptions and develop efficient algorithms for \texttt{CILP}.



\textbf{Using the reduced cost optimality condition}: ~\citet{sun2023maximum} proposed a method to use the reduced cost optimality conditions ~\cite{luenberger1984linear} for LPs. The method constructs a surrogate loss function that encourages the prediction to satisfy the reduced cost optimality conditions. The resulting method has theoretical convergence guarantees, assuming that the LPs are non-degenerate and that the model can interpolate the training data. Both of these are strong assumptions and are not necessarily satisfied in practice. In contrast, we use the KKT conditions that are equivalent to the reduced cost optimality conditions for non-degenerate LPs (see~\cref{app:kkt_equiv} for proof). Since the KKT conditions can also be used for degenerate LPs, our proposed framework provides theoretical guarantees without relying on this assumption. Moreover, our guarantees hold even without assuming interpolation. 

\textbf{Differentiating through LPs}: ~\citet{vlastelica2019differentiation} estimate the gradient ``through'' the LP by calculating the change in the decision by perturbing the prediction. However, it introduces additional hyper-parameters that are non-trivial to tune. Another common technique is to use the straight-through-estimator (ST)~\citep{sahoo2022backpropagation}. Given a set of predictions from the model, the ST method uses the LP to estimate the decisions. However,  it does not consider the LP (treats the corresponding Jacobian as an identity matrix) while back-propagating the gradient from the decisions to the model parameters. Though successful in practice, this method is not theoretically principled.  The method in~\citet{berthet2020learning} computes expected gradients by perturbing the prediction target in different directions. While this method accurately models the gradient, it is not practically feasible because of the computational cost of solving LPs multiple times for each update to the model. One advantage of these techniques is their ``black-box'' nature, meaning that they only rely on the outputs from an LP, thus allowing the use of faster problem-specific solvers. In contrast, our work leverages LP properties (and some other problems described in~\cref{sec:extension-nonlinear-constraints}), allowing us to develop a more efficient and principled approach.


\textbf{Implicit Differentiation}: The methods in~\citep{amos2017optnet, amos2019differentiable} focuses on (strongly)-convex optimization problems. It calculates the gradient through such problems by differentiating through its optimality (KKT) conditions. However, since the solutions of LPs are located at the corners of the feasible polytope, this method will yield zero gradients for LPs. To address this, QPTL ~\citep{wilder2019melding} add a quadratic regularization to the LP, thus relaxing the problem to a non-linear strongly-convex quadratic program (QP) and then use the technique in~\citet{amos2017optnet}. Similarly,~\citet{cameron2022perils} relax Mixed Integer Programs (MIP) by using a log-barrier regularization followed by the use of the techniques in~\citet{amos2017optnet}. These approaches suffer from two notable limitations: (i) they introduce additional hyper-parameters (the regularization strength), and (ii) they are only guaranteed to converge in the vicinity of the optimal solution (because of the bias introduced by the regularization). While our proposed framework also uses KKT conditions, it ensures convergence to the optimal solution without introducing additional hyper-parameters.

\textbf{Predict and Optimize}: The CIO problem is related to the \emph{predict and optimize} (PO) framework~\citep{elmachtoub2022smart}. In contrast to the CIO, PO requires the knowledge of ground-truth costs. Our work does not assume access to this additional information, but we note that the proposed algorithms can be directly used in the PO setting.    

In the next section, we formally formulate the problem and highlight the technical challenges. 





\vspace{-2ex}
\section{Problem Formulation}
\label{sec:problem}
We consider the optimization procedure to be a linear program (LP). Without loss of generality, we assume the standard form of the LP and define $\hat{x}(c)$ as the solution to the LP with cost-vector $c \in \mathbb{R}^{m}$, 
\begin{align*}
\hat{x}(c) := \arg \min_{x} \langle c, x \rangle \text{ s.t } Ax=b,\;x\geq0 \,,
\end{align*}
where $A \in \mathbb{R}^{n \times m}$ and $b \in \mathbb{R}^{n}$. The \texttt{CILP} problem consists of a training dataset $\cD = \{z_i, x_i^*\}_{i=1}^N$ where $z_i \in \mathbb{R}^{1 \times d}$ is the input ($Z \in \mathbb{R}^{N \times d}$ is the corresponding feature matrix) and $x_i^* \in \mathbb{R}^m$ is the corresponding optimal decision. We assume that the LP parameters ($A$, $b$) encoding the constraints are known, whereas the cost vector $c$ is unknown and will be predicted using the data. 


{\textit{Example}: In the context of the shortest path problem (Example 2 in~\cref{sec:introduction}), consider an arbitrary $x, c \in \mathbb{R}^m$ in the dataset; for all \(j \in [m]\), $x_j^* \in \{0,1\}$ variables denote whether an edge is included in the shortest path and the weight of each edge is represented by the cost $c_j$. To ensure a valid path from the start vertex \(s\) to the target vertex \(t\), the ``flow'' constraints are encoded via \(A\) and \(b\). These constraints ensure that every vertex, except \(s\) and \(t\), maintains an equal number of incoming and outgoing edges. Vertex \(s\) is constrained to have exactly one outgoing edge, and vertex \(t\) has precisely one incoming edge. 

When using a model $f_\theta$ with parameters $\theta$ to predict the cost-vector, we define $\hat{x}(\hat{c})$ as: 
\begin{align*}   
\hat{x}(\hat{c}) := \arg \min_{x} \langle \hat{c}, x \rangle \text{ s.t } Ax=b,\;x\geq0, \; \hat{c} = f_{\theta}(z) \,.
\end{align*}
Given the dataset $\cD$ and knowledge of $(A, b)$, the \texttt{CILP} objective is to learn $\theta$ s.t. $\hat{x}(f_\theta(z_i)) \approx x^*_i$ for all $i \in [N]$. 

\vspace{-1ex}
\subsection{Challenge in Gradient Estimation}
\label{sec:challenges-grad-estimation}
To gain some intuition as to why the typical end-to-end learning approach via auto-differentiation~\cite{NEURIPS2019_9015} will not work for \texttt{CILP}, consider using the squared loss to quantify the discrepancy between $\hat{x}(f_\theta(z_i))$ and $x^*_i$, i.e. $\ell(\theta) := \frac{1}{2} \sum_{i=1}^N \ \| \hat{x}(f_\theta(z_i)) - x^*_i\|^2$. Using the chain rule to compute the gradient with respect to $\theta$, we get that $\frac{\partial l}{\partial \theta} = \frac{\partial l}{\partial x} \, \frac{\partial x}{\partial c} \, \frac{\partial c}{\partial \theta}$. The first and last terms can be easily calculated. However, for an LP, the decision $x$ is piece-wise constant with respect to $c$, and $\frac{\partial x}{\partial c}$ is either $0$ or undefined. 

Consequently, in the next section, we aim to develop an algorithm that does not rely on directly calculating $\frac{\partial x}{\partial c}$. 

\vspace{-2ex}
\section{Reduction to Convex Feasibility}
\label{sec:feasibility}
For a linear model, we reduce the \texttt{CILP} problem to convex set feasibility (\cref{sec:feasibility-reduction}). In~\cref{sec:algorithm}, we use alternating projections onto convex sets (POCS) to solve the feasibility problem and completely instantiate~\cref{alg:revgrad}. In~\cref{sec:practical}, we describe some practical considerations when using~\cref{alg:revgrad}. In~\cref{sec:extension-nonlinear-constraints}, we describe how to extend these ideas to handle non-linear but convex objectives and constraints.

\vspace{-1ex}
\subsection{Reduction}
\label{sec:feasibility-reduction}
Recall that for an input $(z, x^*) \in \cD$, we aim to find a $c$ such that $\hat{x}(c) = x^*$. However, due to the non-uniqueness of the mapping from $x$ to $c$, there are potentially infinitely many values of $c$ that can yield $x^*$. We define $C$ to represent the set encompassing all such values of $c$. The set $C$ can be represented by exploiting the optimality conditions for the LP. KKT conditions \cite{kuhn1951nonlinear} give necessary and sufficient conditions for the optimality of the LP. If $x^*$ is the solution to the standard LP, then the KKT conditions can be written as follows:
 \begin{align*}
{\nu}^T A + \lambda - c=0 \,, x^* \cdot \lambda =0 \,, \lambda \geq 0 \,, Ax^*=b \,, x^* \geq 0
 \end{align*} 
where $\lambda \in \mathbb{R}_{+}^{m}, \nu \in \mathbb{R}^{n}$ are the dual variables, $x^*_i \lambda_i =  0$ (for all $i$) represents the complementary slackness condition and $Ax^*=b$, $x^* \geq 0$ represents the feasibility of $x^*$. At optimality, there exist dual variables $(\lambda, \nu)$ such that the tuple ($x^*, \lambda, \nu, c$) satisfies the KKT conditions. 
 
Since the KKT optimality conditions are both necessary and sufficient, given an optimal solution $x^*$, we can identify the set of cost vectors $c$ that satisfy these conditions. This enables us to define the convex set $C$ as:
\begin{align}
   C =\{c\, |\, \exists\, \lambda, \nu \, \text{ s.t. } & \nu^TA+\lambda-c = 0, \nonumber \\ & x^*\cdot \lambda =  0, \lambda \geq 0 \}
\label{eq:c-set-def}
\end{align}   
Note that we omit the condition $Ax^*=b$, $x^* \geq 0$ as it is satisfied by definition for the optimal solution $x^*$. For any $\lambda, \nu$, the set $C$ is affine and hence convex in $c$. 

We define $F$ as the set of cost vectors that are realizable by the linear model parameterized by $\theta \in \mathbb{R}^{d \times m}$. Formally, $F$ can be written as:
\begin{align}
    F = \{c \,| \exists \theta \text{ s.t. } c = z\theta \}.  
\label{eq:f-set-def}    
\end{align}
For a linear model, the set $F$ is linear and hence convex in $c$. The objective of \texttt{CILP} is to find a $c \in C$ (resulting in the optimal solution $x^*$) and is also realizable by the model, i.e. it lies in set $F$. Hence, we aim to find a $c$ that lies in the intersection $(C \cap F)$. Therefore, \texttt{CILP} is equivalent to a convex feasibility problem in this setting. 
\vspace{-2ex}
\subsection{Algorithm}
\label{sec:algorithm}
The commonly employed method for solving convex feasibility problems is the alternating projections (POCS) algorithm~\citep{von1949rings, bauschke1996projection}. The POCS algorithm alternatively projects a point onto the two sets. The algorithm is guaranteed to converge to a point in the intersection if the intersection is non-empty; otherwise, it converges to the closest point between the two sets~\citep{deutsch1984rate, bauschke1993convergence}. Moreover, the rate of convergence is linear in the number of POCS iterations. In order to use the POCS algorithm for the \texttt{CILP} problem, we require the projection of an arbitrary point $q \in \R^{m}$ onto the set $C$. This corresponds to solving the quadratic program (QP) as follows:
\begin{subequations}
\label{grad_inverse_lp_new}
\begin{align*}
\mathcal{P}_{C}(q) & := \underset{c}{\arg \min} ||c - q||^2_{2} \\
\text{subject to}\qquad & \nu^T A-c+ \lambda =0, \lambda \cdot x^* = 0,\lambda \geq 0
\tag{\ref{grad_inverse_lp_new}} 
\end{align*}
\end{subequations}
\cref{grad_inverse_lp_new} returns a point $\mathcal{P}_{C}(q)$, the Euclidean projection of $q$ onto $C$. For the projection of a point $q$ onto the set $F$, we require solving the following regression problem, 
\begin{align}
\hat{\theta} := \arg \min_{\theta} \frac{1}{2} \, || q - z\theta ||^{2} \quad \text{;} \quad \mathcal{P}_{F}(q) = z\hat{\theta}
\label{eq:proj-f}
\end{align}
\cref{eq:proj-f} returns a point $\mathcal{P}_{F}(q)$, the Euclidean projection of $q$ onto $F$. Hence, POCS consists of alternatively solving the optimization problems in~\cref{grad_inverse_lp_new} and~\cref{eq:proj-f}.
\vspace{1ex}
\begin{algorithm}[!h]
\caption{for \texttt{CILP}} 
\label{alg:revgrad}
\textbf{Input}: $A, b$, Training dataset $\mathcal{D} \equiv (z_i,x_i^*)_{i=1}^N$, Model $f_{\theta}$ \; \\
Initialize $\theta_1$\; \\
  \For {$t =1,2, ..,T$} {
    $\hat{c}_i = f_{\theta_t}(z_i) $, $\forall i \in [N]$ \;
    
    \For {$i = 1, 2, .., N$}{
        $q_i = \mathcal{P}_{C_i}(\hat{c}_i)$ by solving the optimization problem in ~\cref{grad_inverse_lp_new} 
        }
    $\theta_{t+1} = \arg \min_{\theta}  \frac{1}{2N} \sum_{i=1}^N|| q_i - f_{\theta}(z_i) ||^{2}$
  }
  \textbf{Output}: $\theta_{T+1}$
\end{algorithm}
\vspace{0.5ex}
In order to extend the above idea to $N$ training points, we will define sets $C_i$ analogous to~\cref{eq:c-set-def} for each $i \in [N]$. We define $\mathcal{C}:= C_1 \times C_2 \ldots \times C_N$ to be the Cartesian product of these sets. Hence, $\mathcal{C}$ consists of concatenated vectors 
$(c_1, c_2, \ldots ,c_N) \in \mathbb{R}^{Nm}$ where $c_i \in C_i$. Similarly, we define $\mathcal{F} := \{(c_1, c_2, \ldots, c_N) \vert \exists \theta \text{ s.t. } \forall i \in [N] \text{ , } c_i = z_i\theta \}$. Hence, the projection onto $\mathcal{C}$ corresponds to solving the QP in~\cref{grad_inverse_lp_new} for every point $i \in [N]$. Projecting an arbitrary point $\tilde{q} = (\tilde{q}_1, \tilde{q}_2, \ldots, \tilde{q}_N) \in \R^{Nm}$ onto $\mathcal{F}$ involves solving the regression problem, $\hat{\theta} := \arg \min_{\theta}  \frac{1}{2N} \sum_{i = 1}^{N} \normsq{\tilde{q}_i - z_i \theta}$. Since the Cartesian product of sets is convex, both $\mathcal{C}$ and $\mathcal{F}$ are convex and hence the resulting POCS algorithm will converge at a linear rate. 

Finally, we note that our algorithmic framework can handle a generic model $f_\theta$, though our theoretical results only hold for a linear model. The complete algorithm for a generic model $f_\theta$ is described in~\cref{alg:revgrad}. Next, we describe practical considerations when implementing the algorithm. 

\vspace{-1.5ex}
\subsection{Practical considerations: Margin}
\label{sec:practical}
\vspace{-0.5ex}
Recently,~\citet{sun2023maximum} have noted the benefits of using a margin with the LP optimality conditions~\citep{luenberger1984linear}. In~\cref{app:kkt_equiv}, we show how to modify their margin formulation for the KKT conditions, resulting in the modified set $C$ below:
\begin{multline}
C =\{c\,\, |\, \exists\, \lambda, \nu \, \text{ s.t. } \, \nu^T A-c+ \lambda =0,  \\ \lambda_i \, \mathcal{I}\{x^*_i \neq 0\}  = 0, \lambda_i \, \mathcal{I}\{x^*_i = 0\} \geq \margin\}
    \label{margin}
\end{multline}
There are two key motivations for adding the margin: (i) it ensures that the algorithm does not converge to a trivial solution corresponding to $c = 0$, (ii) it ensures that the algorithm will converge to the interior (rather than the boundary) of $C$ making it more robust to perturbations and improving the algorithm's generalization performance~\cite{el2019generalization} to previously unseen instances. We note that our framework and the resulting algorithm is not limited to LPs. Next, we describe how to extend these ideas to handle non-linear but convex objectives and constraints.  

\vspace{-2ex}
\subsection{Handling non-linear optimization problems}
\label{sec:extension-nonlinear-constraints}
Similar to the linear case, the KKT optimality conditions can be used to derive a convex set $C$ for a class of non-linear convex objectives and constraints~\cite{iyengar2005inverse}. As an example, we instantiate $C$ for a specific quadratic program (QP) below and defer the general case to~\cref{app:proofs}. 
\begin{align}
\hat{x}(c) := \arg \min_{x \geq 0} -\langle c, x \rangle + \frac{\gamma}{2} x^{T} Q x \text{ s.t } \sum_{i=1}^{m} x_i = 1 
\label{eq:qp-def}
\end{align}
\emph{Example}: For the portfolio optimization problem~\citep{fabozzi2008portfolio} common in econometrics, $x_j$ and $c_j$ in~\cref{eq:qp-def} represent the fraction of investment and the expected return for stock $j \in [m]$ respectively. The matrix $Q \in \mathbb{R}^{m \times m}$ represents the risk associated with selecting similar stocks, and $\gamma$ is the given risk tolerance. The task is to maximize the return while minimizing the risk, subject to simplex constraints. Given historical data and a model that can be used to predict the expected return and risk matrix, and the ``best'' portfolios in hindsight, the CIO problem is to infer the model parameters. In this case, the convex set $C$ consisting of $\chi := \{c^*, Q^*\}$ that satisfy the KKT conditions of the QP is given as: $C  =\{c, Q \, | \, \exists\, \lambda, \nu \, \text{ s.t. } \,  \nu \mathbf{1}_m + \lambda + c - \frac{\gamma}{2} (Q + Q^T)x^*=0,\, x^*\cdot \lambda =  0, \, \lambda \geq 0 \}$ where $\lambda \in \mathbb{R}_{+}^{m}, \,\nu \in \mathbb{R}^{1}, \,\mathbf{1}_m = (1, 1, \dots, 1)$.  The set $C$ is linear and therefore convex in $\{c, Q\}$. Consequently, when using a linear prediction model, the inverse problem for portfolio optimization can also be reduced to convex feasibility. 

For a general non-linear convex objective $\phi(x,\omega)$ where $\omega$ represents a general cost vector, set $C$ is convex in $\omega$ if $\frac{\partial {\phi(x,\omega)}}{\partial x}|_{x=x^*}$, the derivative of the objective function $\phi$ w.r.t $x$ evaluated at $x^*$ is convex in $\omega$. For instance, this condition also applies to semi-definite programs. Finally, we note that our framework is not limited to linear constraints, and can easily handle non-linear convex constraints. Please refer to~\cref{app:gen_kkt_inverse} for the derivation. 



\vspace{-2ex}
\subsection{Challenges for solving large-scale problems}
\label{sec:challenges-large}
The POCS approach described above requires computing the exact projection of point $q$ onto the set $\mathcal{F}$. For high-dimensional problems, computing these exact projections is computationally expensive. Moreover, for non-linear models such as neural networks, $\mathcal{F}$ is non-convex and the resulting projection is ill-defined. Additionally, computing the exact projection requires iterating through the entire dataset of $N$ points, which can be prohibitive for large datasets typical in practice. 

Consequently, in the next section, we reduce the problem to empirical risk minimization on an appropriate smooth, convex loss satisfying the PL condition. This reduction enables the use of computationally efficient (stochastic) first-order optimization algorithms common in the machine learning literature. 


\vspace{-2ex}
\section{Reduction to Empirical Risk Minimization}
\label{sec:erm}
For a linear model, in~\cref{sec:erm-reduction}, we reduce the feasibility problem to empirical risk minimization (ERM) on an appropriate smooth, convex loss satisfying the PL condition and prove that the preconditioned gradient method (with a specific preconditioner) on this loss is equivalent to the POCS approach of~\cref{alg:revgrad}. Subsequently, in~\cref{sec:first-order}, we consider using computationally efficient (stochastic) first-order methods for minimizing the loss functions. 

\vspace{-1.5ex}
\subsection{Reduction}
\label{sec:erm-reduction}
We define the loss function $h(\theta)$ as follows:
\begin{align}
h(\theta) := \frac{1}{2N} \sum_{i=1}^N \min_{q_i \in C_i} \| f_{\theta}(z_i) - q_i\|^2 \,,
\label{eq:train-loss}
\end{align}
where $C_i$ represents the set of feasible cost vectors (defined in~\cref{eq:c-set-def}) for data-point $i$. Hence, $h(\theta)$ represents the mean (across the data-points) of squared distances between the predicted cost vector $f_{\theta}(z_i)$ and the set $C_i$. In order to better interpret $h(\theta)$, consider a point $c_\theta = (c_1, c_2, \ldots, c_N) \in \R^{Nm}$ such that $c_i = f_\theta(z_i)$. Hence, $h(\theta) = \frac{d^{2}(c_\theta, \mathcal{C})}{N}$ where $d^{2}(w,\mathcal{W})$ is the squared Euclidean distance of point $w$ to the set $\mathcal{W}$. Since $c_\theta \in \mathcal{F}$, minimizing $h(\theta)$ is related to minimizing the distance between the sets $\mathcal{F}$ and $\mathcal{C}$. Formally, in~\cref{thm:conv_proof} (proved in~\cref{proof:conv_proof}), we can reduce the feasibility problem in~\cref{sec:feasibility} to minimizing $h(\theta)$. 

\begin{restatable}{proposition}{propconverge}
Point $\hat{c} := (c_1, c_2, \ldots, c_N)$ where $c_i = z_i \tilde{\theta}$ and $\tilde{\theta} \in \arg\min h(\theta)$ lies in the intersection $\mathcal{C} \cap \mathcal{F}$ if it exists, else $\hat{c} \in \mathcal{F}$ is the point closest to $\mathcal{C}$. 
\label{thm:conv_proof}
\end{restatable}

\vspace{-2ex}
\subsection{Properties of $h(\theta)$}
\label{sec:prop-h}
For a linear model where $f_\theta(z) = z\theta$, we show that $h(\theta)$ has desirable properties that allow it to be minimized efficiently. In the proposition below, we establish the convexity and smoothness of $h(\theta)$. 

\begin{restatable}{proposition}{propsmooth}
For a linear model $f_{\theta}(z) = z\theta$ parameterized by $\theta \in \R^{d \times m}$, assuming (without loss of generality) that $\forall i$, $\|z_i\|\leq 1$, $h(\theta)$ is a $1$-smooth convex function.
\label{thm:smooth-convex}
\end{restatable}
The proof of the above proposition is included in~\cref{proof:smooth-convex}. In addition to convexity, we prove that when using a linear model, $h(\theta)$ satisfies the Polyak-Lojasiewicz (PL) inequality~\citep{polyak1964gradient,karimi2016linear}. The PL condition is a gradient domination property that implies curvature near the minima and entails that every stationary point is a global minimum. Formally, the PL inequality states that there exists a constant $\mu > 0$ such that for all $\theta$, 
\begin{align}
h(\theta) - h^* \leq \frac{1}{2 \mu} || \nabla h(\theta) ||^{2} \,,
\label{eq:pl-def}
\end{align}
where $h^*$ is the minimum of $h$.

\begin{restatable}{proposition}{proppl}
For a linear model $f_{\theta}(z) = z\theta$ and assuming (i) (without loss of generality) that $\forall i$, $\|z_i\|\leq 1$ and (ii) $\lambda_{\min}[Z^T Z] > 0$, $h(\theta)$ is not necessarily strongly-convex but satisfies the PL inequality with $\mu = \lambda_{\min}\left[\frac{\sum_{i = 1}^{N} z_i z_i^\top}{N}\right]$. 
\label{thm:pl}
\end{restatable}



We include the proof in~\cref{proof:pl} and note that such a result showing that square distance functions to (non)-convex sets is PL was also recently shown in~\citep{garrigos2023square}. Importantly, we note that convexity coupled with the PL condition implies that the function satisfies the restricted secant inequality (RSI)~\citep{zhang2013gradient}, a stronger condition than PL but weaker than strong-convexity~\citep[Theorem 2]{karimi2016linear}. 

Since we have reduced the \texttt{CILP} to a problem of minimizing a loss function with desirable properties, we can use computationally efficient techniques like gradient descent and its stochastic and adaptive~\citep{kingma2014adam,duchi2011adaptive} variants.

\vspace{-1ex}
\subsection{First-order Methods}
\label{sec:first-order}
We first show that for a linear model,~\cref{alg:revgrad} is equivalent to the preconditioned gradient method on $h(\theta)$. With $Z \in \R^{N \times d}$ being the feature matrix, the preconditioned gradient update for minimizing $h(\theta)$ at iteration $t \in [T]$ with step-size $\eta$ and the preconditioner equal to $\left[\frac{Z^TZ}{N} \right]^{-1}$, is given as:
\begin{equation}
    \theta_{t+1} = \theta_t - \eta \, [{Z^TZ}]^{-1} \, Z^T(Z\theta_t - q_t) \, ,
    \label{eq:pgd-linear}
\end{equation}
where $q_t = \mathcal{P}_{\mathcal{C}}(Z\theta_t)$ and $\|Z\theta_t - q_t\|$ is the Euclidean distance to the set $\mathcal{C}$ at iteration $t$. Consequently, point $Z \theta_{t+1}$ is exactly the Euclidean projection of $q_t$ onto the set $\mathcal{F}$. In~\cref{proof:pocs-pgd}, we prove the following proposition.

\begin{restatable}{proposition}{proppocssgd}
For a linear model $f_{\theta}(z) = z\theta $, the iterates corresponding to the preconditioned gradient method on $h(\theta)$ with $\eta = 1$ are identical to~\cref{alg:revgrad}. 
\label{thm:pocs-pgd}
\end{restatable}

We note that for general non-linear models, this connection to POCS and hence~\cref{alg:revgrad} does not necessarily hold.

Next, we consider minimizing $h(\theta)$ using gradient descent (GD) with step-size $\eta_t$ at iteration $t$. This results in the following general update:
\begin{align}
    \theta_{t+1} = \theta_t - \eta_t \frac{\sum_{i = 1}^{N} \frac{\partial f_{\theta}(z_i)}{\partial \theta}\vert_{\theta = \theta_t}  \, [f_\theta(z_i) - q_{i,t}]}{N}
    \label{eq:gd}
\end{align}
where $q_{i,t} = \mathcal{P}_{C_i}(f_{\theta_{t}}(z_i))$ and $\frac{\partial f_\theta(z_i)}{\partial \theta}\vert_{\theta = \theta_t}$ is the Jacobian of $f_\theta$ at iterate $\theta_t$. For a linear model, this simplifies to:
\begin{align}
    \theta_{t+1} = \theta_t - \frac{\eta_t}{N} \left[Z^{T} (Z\theta_t - q_t)\right] \,,
       \label{eq:gd-linear}
\end{align}
where $q_t = \mathcal{P}_{\mathcal{C}}(Z\theta_t)$ and $\|Z\theta_t - q_t\|$ is the Euclidean distance to the set $\mathcal{C}$ at iteration $t$. The update in~\cref{eq:gd-linear} can be interpreted as an inexact projection of $q_t$ onto $\mathcal{F}$. 

If \( \tilde{\theta} \in \arg\min_{\theta} h(\theta) \), standard convergence results~\citep{karimi2016linear} for smooth, convex and PL loss functions guarantee that GD, after \( T \) iterations, returns \( \theta_T \) such that \( h(\theta_T) - h(\tilde{\theta}) = O(\exp(-T)) \). To illustrate what this convergence rate implies for the feasibility and consequently the \texttt{CILP} problem, consider the case where $\mathcal{C} \cap \mathcal{F}$ is non-empty. In this case,~\cref{thm:conv_proof} guarantees that $h(\tilde{\theta}) = 0$ and hence, using GD with $T = O\left(\ln(\sqrt{N}/\epsilon) \right)$ iterations is guaranteed to return a point $\hat{c} := (c_1, c_2, \ldots, c_N) \in \mathcal{F}$ where $c_i = z_i \theta_{T}$ that is $\epsilon$ close to $\mathcal{C}$. Compared to~\cref{alg:revgrad}, we see that GD retains the fast linear convergence rate to a point in $\mathcal{C} \cap \mathcal{F}$.   

GD requires iterating through the entire dataset for each update, which is inefficient for large datasets. To address this, we use stochastic gradient descent (SGD)~\citep{robbins1951stochastic}. Writing \( h(\theta) = \frac{1}{N} \sum_{i = 1}^{N} h_i(\theta) \) where \( h_i(\theta) = \frac{1}{2} \min_{q_i \in C_i}\| f_{\theta}(z_i) - q_i\|^2 \), the SGD update with step-size \( \eta_t \) at iteration \( t \) is:

\begin{align}
\vspace{-2ex}
\theta_{t+1} = \theta_t - \eta_t \frac{\partial f_{\theta}(z_{i_t})}{\partial \theta}\bigg\vert_{\theta = \theta_t}  \, [f_\theta(z_{i_t}) - q_{i_t,t}]\,,
\label{eq:sgd}
\end{align}
where $i_t \in [N]$ is the index of the loss function sampled uniformly at random at iteration $t$. For a linear model, 
\begin{align}
\vspace{-2ex}
\theta_{t+1} = \theta_t - \eta_t \, z_{i_t}^T (z_{i_t} \theta_t - q_{i_t,t}) \,,
\label{eq:sgd-linear}
\end{align}
where $q_{i_t,t} = \mathcal{P}_{\mathcal{C}_i}(z_{i_t} \theta_t)$. Similar to GD, the update in~\cref{eq:sgd-linear} can be interpreted as an inexact projection of $q_{i_t}$ onto $C_i$. Compared to GD, which has an $O(N)$ per-iteration cost, SGD has an $O(1)$ iteration cost, making it preferable for large datasets. However, in general, SGD has a slower rate of convergence compared to GD. Specifically, when minimizing smooth, convex functions and PL functions, $T$ iterations of SGD with a decreasing $O(1/T)$ step-size is guaranteed to return $\theta_{T}$ such that $\E[h(\theta_T)] - h(\tilde{\theta}) = O(1/T)$~\citep{karimi2016linear,gower2021sgd}, where the expectation is over the random sampling in each iteration. 

If an additional \emph{interpolation} property is satisfied, SGD with a constant step-size can match the convergence rate of GD~\citep{ma2018power,vaswani2019fast,bassily2018exponential,raj2021explicit}. Formally, for convex loss functions, interpolation is satisfied when $\tilde{\theta} := \arg\min h(\theta)$ also simultaneously minimizes each $h_i$, i.e. $||\nabla h_i(\tilde{\theta})|| = 0$ for all $i \in [N]$. In the context of the feasibility problem, interpolation is satisfied if $f_{\tilde{\theta}}(z_i) \in C_i$ and hence $h_i(\tilde{\theta}) = 0$  for all $i \in [N]$. This implies that the intersection $\mathcal{C} \cap \mathcal{F}$ is non-empty and in this case, SGD with a constant step-size requires $T = O\left(\ln(\sqrt{N}/\epsilon) \right)$ iterations to return a point $\epsilon$-close to $\mathcal{C} \cap \mathcal{F}$. Notably, the PL condition is not required for convergence; smooth and convex functions (even without the PL condition) ensure convergence with first-order methods, though at a slower rate (e.g., $O(1/\sqrt{T})$ instead of $O(1/T)$ for SGD).

The above results hold when using a linear model, ensuring convexity in the resulting function $h(\theta)$. Similar guarantees extend to non-parametric techniques like kernel methods, demonstrating the generality of our results. However, for expressive models such as deep neural networks, convexity is not necessarily satisfied. In certain regimes of over-parametrized neural networks, conditions resembling PL or variations thereof are satisfied~\citep{liu2022loss,liu2023aiming}. In these cases, SGD can still achieve linear convergence~\citep{vaswani2019fast,bassily2018exponential}, matching the results in the convex case.


The above results are concerned with minimizing the loss on the training dataset. In the next section, we study the generalization performance of SGD on previously unseen instances sampled from the same distribution. 

\vspace{-2ex}
\section{Generalization Guarantees}
\label{sec:generalization}
In this section, we use the existing results on algorithmic stability~\citep{bousquet2002stability,hardt2016train,lei2020fine} to control the generalization error and subsequently bound the suboptimality for \texttt{CILP}. 

We first define the necessary notation and recall the necessary results from the algorithmic stability literature. We define $\rho$ to be the probability measure on the sample space $\mathcal{Y} = \mathcal{Z} \times \mathcal{X^*}$, where  $\mathcal{Z} \subseteq \mathbb{R}^d$ and $\mathcal{X^*} \subseteq \mathbb{R}^m$. We assume that the training dataset $\cD =\{(z_1,x^*_1), \cdots, (z_N, x^*_N)\}$, is drawn independently and identically from $\rho$. 
We define $h(\theta,(z,x^*)) := \frac{1}{2}\min_{q \in C(x^*)} \|f_{\theta}(z) - q\|^2$, where $C(x^*)$ is the set constructed according to~\cref{eq:c-set-def}. Furthermore, we denote the \emph{population loss} for parameter $\theta$ as: $\hat{h}(\theta) = \mathbb{E}_{(z,x^*) \sim \rho} [h(\theta, (z,x^*))]$.  

Based on algorithmic stability,~\citet{lei2021sharper} prove the following generalization result for learning with smooth loss functions satisfying the PL condition. 

\begin{restatable}{theorem}{stabilityres}(Theorem 1 in \citep{lei2021sharper}) \label{thm:stability_grad_dom}
    Let $\theta_{\cD}$ denote the output of a randomized algorithm $\mathcal{A}$ when minimizing an $L$-smooth function $h$ that satisfies PL inequality with constant $\mu$. Under the condition $N \geq \nicefrac{4 L}{\mu}$, we have,
   \begin{equation} 
    \begin{split} \label{eq:1}
    \mathbb{E}[ \hat{h}(\theta_{\cD}) - \inf_\theta h(\theta) ] 
    &= O\left(\frac{\mathbb{E}[\inf_{\theta} h(\theta)]}{N\mu} \right. \\
    &\quad \left. + \frac{\mathbb{E}[h(\theta_{\cD}) - \inf_{\theta} h(\theta)]}{\mu}\right)
    \end{split}
\end{equation}
\end{restatable}
\vspace{-2.5ex}
The expectation in the above theorem is w.r.t the randomness in selecting the training dataset of size $N$ and w.r.t the stochasticity in the learning algorithm. In our context, since the randomized algorithm $\mathcal{A}$ is SGD, the bound on its generalization is a direct consequence of  Theorem~\ref{thm:stability_grad_dom}. In particular, since $\E_{\cD}[\inf_\theta h(\theta)] \leq \inf_\theta \E_{\cD} [h(\theta)] = \inf_\theta \hat{h}(\theta)$, we can obtain the following result from~\citet{lei2021sharper}.


\begin{restatable}{corollary}{stabilitycorrsi}
(Theorem 6 in \citep{lei2021sharper})
When minimizing an $L$-smooth, $\mu$-RSI function, SGD with step-size $\eta_t = \frac{1}{\mu(t+1)}$ for all $t>0$ has the following guarantee, 
\begin{align*}
        \mathbb{E}[\hat{h}(\theta_T)] - \inf_{\theta} \hat{h}(\theta) = O\left( \frac{1}{N \mu} + \frac{1}{\mu^2 T}\right).
    \end{align*}
\label{cor:no-interpolation}    
\end{restatable}
\vspace{-3ex}
The expectation in the above result is only over the stochasticity in SGD. The LHS represents the \emph{excess risk}, while the first term on the RHS decreases as $N$ increases, and the second term on the RHS represents the average (over $\cD$) optimization error that decreases as $T$ increases. 

In the interpolation setting, since the model can fit \textit{any} training dataset of size $N$ using SGD, $\inf_\theta \E[h(\theta)] = 0$. In this case, we obtain the following result from~\citet[Theorem 7]{lei2021sharper}. 
\begin{restatable}{corollary}{stabilitycorint} 
When minimizing an $L$-smooth, $\mu$-RSI function and if $\inf_\theta \E[h(\theta)] =0$ for any choice of training dataset $\cD$ of size $N$, SGD with step-size $\eta_t = \eta = \frac{1}{L}$ for all $t>0$ has the following guarantee,
    \begin{align*}
        \mathbb{E}[\hat{h}(\theta_T)] = O\left(\frac{L(1-\frac{\mu}{L })^T}{2\mu} \right).
    \end{align*}
\label{cor:interpolation} 
\vspace{-4ex}
\end{restatable}
The above result shows that in the interpolation setting, the expected (over the randomness in SGD) population loss decreases at a linear rate (depending on $T$). Importantly, the above bound does not depend on $N$. Intuitively, if $h(\theta)$ is smooth and $N$ is large enough s.t. it satisfies the PL condition with $\mu > 0$, minimizing the loss over a single dataset results in minimizing the population loss. 

The above results bound the population loss $\hat{h}(\theta)$, which serves as a proxy for the decision quality of SGD. Subsequently, we establish a connection between $h(\theta)$ and the suboptimality in the CIO framework. 

\vspace{-2ex}
\subsection{Sub-optimality}
\label{sec:subopt-gen}
In this section, we first argue about the shortcomings of previous definitions of sub-optimality to measure performance for\texttt{CILP} and then propose a new sub-optimality metric.

Recently, \citet{sun2023maximum} define the suboptimality gap as $\Gamma_{1}(\theta, (z,x^*)) := \langle c_{\theta}, \,x^*- \hat{x}(c_{\theta}) \rangle$ where $c_{\theta} = f_{\theta}(z)$, and prove theoretical guarantees for this loss. We argue that $\Gamma_{1}(z,x^*)$ is not the right metric as the predicted $c_{\theta}$ can be made arbitrarily small, resulting in smaller values of $\Gamma_{1}(z,x^*)$ without ensuring that $x^* \approx \hat{x}(c_{\theta})$. On the other hand, work in predict-and-optimize~\citep{elmachtoub2022smart} assumes access to the ground-truth cost-vector $c^*$ and proposes to use a sub-optimality $\Gamma_{2}(\theta, (z, c^*, x^*)) := \langle c^*, \, \hat{x}(c_{\theta}) - x^* \rangle$. Since we do not have access to 
$c^*$, we cannot directly use this measure of sub-optimality. Consequently, we use the projection of $c_\theta$ onto $\cC$ as a proxy for the ground-truth $c^*$ and define the suboptimality gap as:
\begin{equation}
    \Gamma(z,x^*) = \left\langle \frac{ \mathcal{P}_{C}(c_{\theta})}{\norm{\mathcal{P}_{C}(c_{\theta})}_{2}}, \, \hat{x}(c_{\theta}) - x^* \right\rangle
    \label{eqn:subopt}
\end{equation}
It is important to note that we divide $\mathcal{P}_{C}(c_{\theta})$ by its corresponding $\ell_2$ norm to make the sub-optimality scale-invariant i.e. small values of $\mathcal{P}_{C}(c_{\theta})$ do not necessarily imply small sub-optimality (unlike $\Gamma_1$). We now relate the sub-optimality to the loss $h(\theta)$ and prove the following result in~\cref{proof:subopt}. 
\begin{restatable}{proposition}{propsubopt}
\label{thm:subopt}
For $c_{\theta} \in \mathbb{R}^m  := f_{\theta}(z)$, assuming that $\forall j \in [m], [\hat{x}(c_\theta)]_j, x_j^* \in [0,1]$, $\Gamma(\theta, (z, x^*)) \leq \frac{\sqrt{2m\, h(\theta)}}{\delta}$ where $\delta := O(\nicefrac{\margin}{\sqrt{m}})$, $\margin$ is the margin and the $O$ notation hides constants that depend on the LP. 
\end{restatable}
As the sub-optimality is upper-bounded by $O(\sqrt{h(\theta)})$, we can control it by controlling the loss $h(\theta)$. Putting together the results in~\cref{cor:interpolation,cor:no-interpolation} and~\cref{thm:subopt}, we observe that $T$ iterations of SGD result in the following bounds on the expected sub-optimality: 
$\E_{(z,x^*) \sim \rho} \left[\E_{\cD \sim \rho} \left[\E[\Gamma(\theta_T, (z, x^*))] \right] \right] = $ $O\left(\frac{\sqrt{2m}}{\delta} \left[\inf_{\theta} \hat{h}(\theta)+ \left[\frac{1}{N \mu} + \frac{1}{\mu^2 T}\right] \right]^{1/2} \right)$ in the general setting and $O\left(\frac{\sqrt{2m}}{\delta} \left[\exp(-\mu \, T) \right]^{1/2} \right)$ (independent of $N$) in the interpolation setting. Compared to these results,~\citet{sun2023maximum} derive an $O(\nicefrac{1}{\sqrt{N}})$ bound on the expected sub-optimality in terms of $\Gamma_{1}$ for both the interpolation and general settings. In the next section, we compare our method against several baselines on real-world and synthetic datasets and present the results.

\section{Experiments\protect\footnote{The code is available~\href{https://github.com/Saurabh-29/Inverse_Optimization_To_Feasibility_To_ERM}{here}}}
\label{sec:experiments}

\begin{figure}[ht]
\begin{center}
\centerline{\includegraphics[width=\columnwidth]{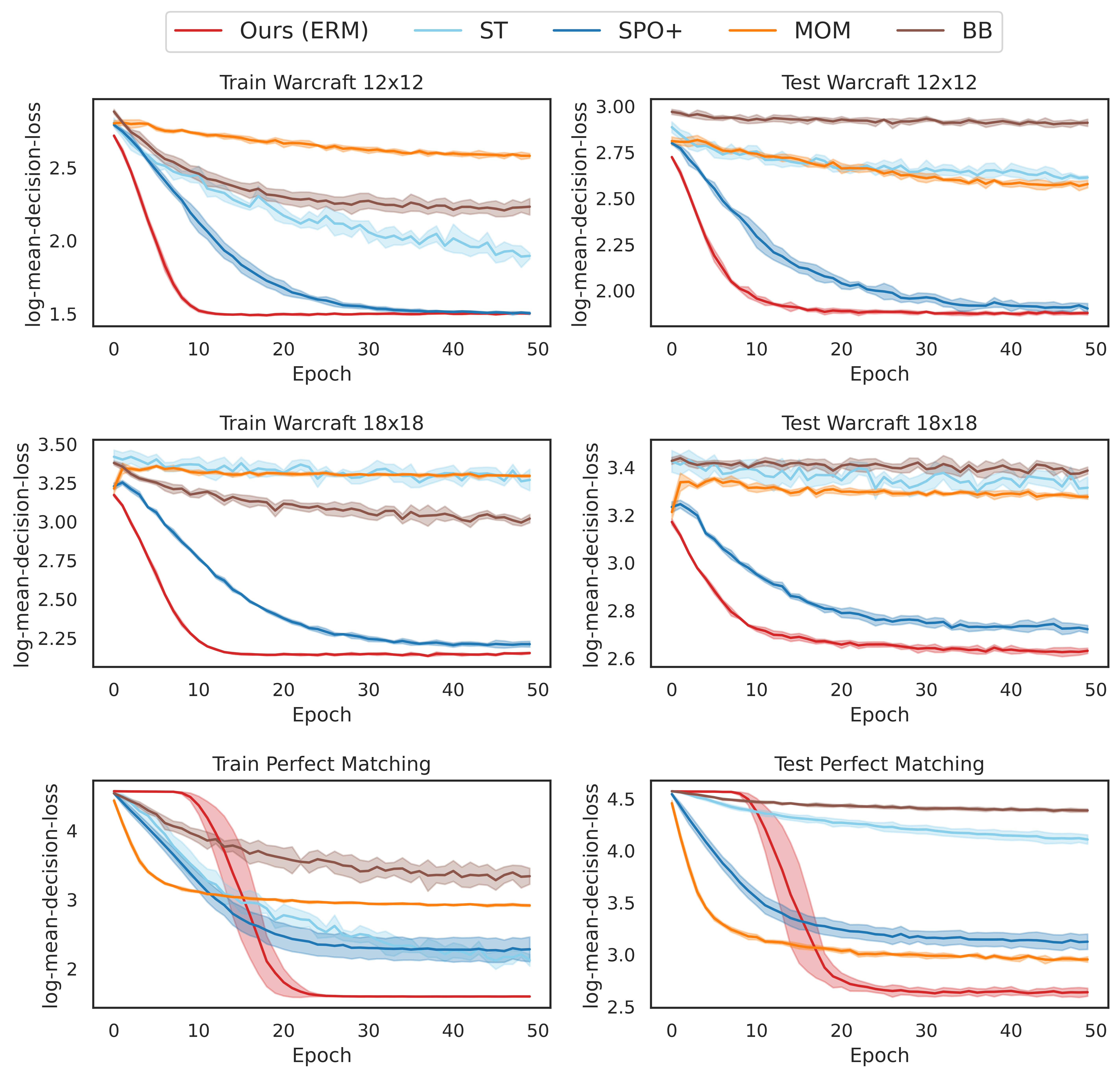}}
\caption{Decision loss: Training and Test plot for the real world experiments.  Our method significantly outperforms the other methods (ST, BB, MOM, SPO+).}
\label{fig:result_res_dec_err}
\end{center}
\end{figure}

\begin{figure}[ht]
\begin{center}
\centerline{\includegraphics[width=\columnwidth]{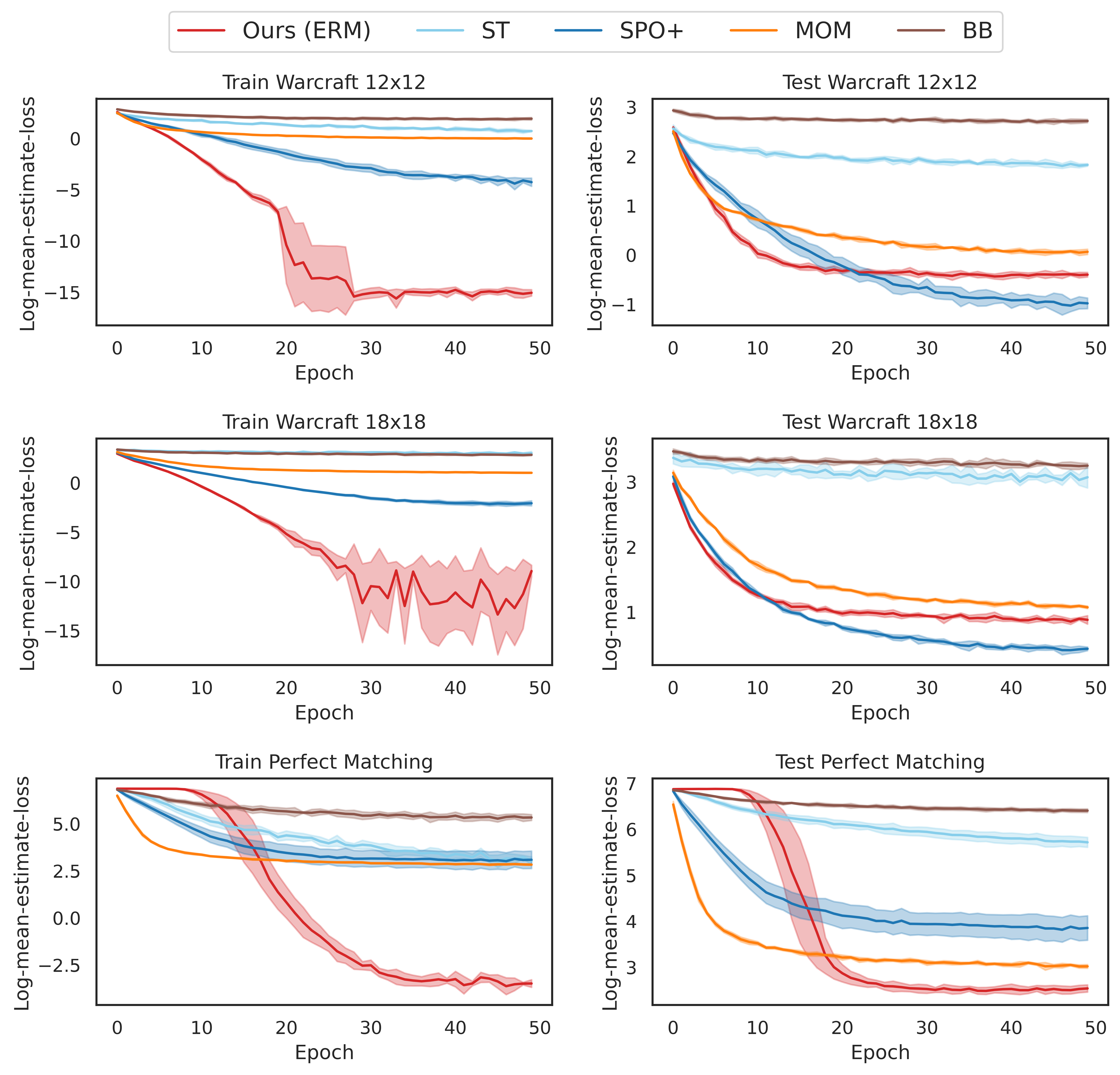}}
\caption{Estimate loss: training and test plots for real-world experiments.  Our method significantly outperforms existing methods (ST, BB, MOM) and is comparable to SPO+, which uses the knowledge of $c^*$.}
\label{fig:result_rw}
\end{center}
\end{figure}

\textbf{Datasets and Model}: To validate the effectiveness of our approach, we experiment with both synthetic and real-world benchmarks. We consider two real-world tasks~\citep{vlastelica2019differentiation} -- Warcraft Shortest Path and Perfect Matching below and defer the synthetic experiments to~\cref{app:synthetic-experiments}. 


\textbf{Warcraft Shortest Path (SP)}: The dataset consists of $(z,x^*)$ pairs where the input $z$ is an RGB image generated from the Warcraft II tileset. The output $x^*$ corresponds to the shortest path between given source-target pairs. The model predicts the edge weights for each tile in a $k \times k$ grid (where $k \in \{12, 18 \}$). Given these edge-weights, the optimization problem is to find the shortest path.

\textbf{Perfect Matching (PM)}: The dataset consists of $(z,x^*)$ pairs where the input $z$ is a grey-scale image consisting of MNIST digits on a $k \times k$ grid. The output $x^*$ is a matching for each digit to one of its neighbors on the grid. The model predicts the edge-weights between each pair of neighbouring digits. Given these edge-weights, the optimization problem is to find a matching that has the minimal cumulative weight of the selected edges. 

Both datasets consist of $10000$ training samples, $1000$ validation samples and $1000$ test samples each. For both SP and PM, we use Resnet-$18$ \cite{he2016deep} followed by a softplus function $s(x) = \log (1+ \exp{(x)})$  to ensure the predicted cost is non-negative. Please refer to~\cref{app:experiments-details} for additional details about the model and datasets. 

\textbf{Methods}: We compare the proposed method against several existing methods, including ST~\cite{sahoo2022backpropagation}, MOM \cite{sun2023maximum}, BB \cite{vlastelica2019differentiation},  QPTL ~\citet{wilder2019melding} and SPO+ \cite{bertsimas2020predictive}. We use adaptive first-order methods: AdaGrad~\citep{duchi2011adaptive} and Adam~\citep{kingma2014adam} to minimize the loss in~\cref{eq:train-loss} for our method and the corresponding losses for the other baselines. We train all the methods for $50$ epochs with a batch size of $100$.  We employ a grid search to find the best constant step size in $\{0.1, 0.05, 0.01, 0.005, 0.001, 0.0005, 0.0001, 0.00005\}$, across both the Adam and Adagrad optimizers. The optimal settings are determined based on performance on the validation set.  For optimal settings, we consider $5$ independent runs and plot the average result and standard deviation. Following~\citet{sun2023maximum}, we set $\margin=1$ for all our experiments. In~\cref{app:margin}, we also provide an ablation study varying $\margin$ in~\cref{tab:margin-ablation} and show that the algorithm is robust to $\margin$. We note that the QPTL approach is excluded from real-world experiments as it is prohibitively slow for large problems~\citep{amos2017optnet, geng2023rethinking}. For all the methods, we implement the LPs and QPs using the CVXPY library~\cite{diamond2016cvxpy}. For LPs, we use the ECOS solver ~\cite{domahidi2013ecos}, and for QPs, we use the OSQP solver ~\cite{stellato2020osqp}. 

\textbf{Metrics}: For each method, we plot the standard metrics: \emph{estimate-loss} and \emph{decision-loss} on both the train and test set, defined as:
\begin{align}
    \text{Estimate-Loss}(\theta) &= \sum_{i=1}^N \langle c_i^*, \hat{x}(f_{\theta}(z_i)) \rangle - \langle c_i^*, x_i^* \rangle \label{eqn:estimate-loss}\\
    \text{Decision-Loss}(\theta) &= \sum_{i=1}^N \|\hat{x}(f_{\theta}(z_i))  - x_i^* \|^2
    \label{eqn:decision-loss}
\end{align}
Since all the datasets consist of $(z_i, c_i^*, x_i^*)$ pairs where $x_i^* = \hat{x}(c_i^*)$, the estimate-loss and decision-loss are commonly used to measure performance in these tasks\footnote{Experimentally, we found that the sub-optimality in~\cref{eqn:subopt} has a similar trend as the estimate-loss.}. We note that SPO+ requires access to the ground-truth cost-vector $c^*$, while other methods, including ours, do not. Though the MOM method does not require access to $c^*$ in principle, the paper's implementation\footnote{See~\href{https://github.com/liushangnoname/Maximum-Optimality-Margin}{MOM code}} uses this ground-truth information to calculate the basis and use the LP optimality conditions. We continue using this information for MOM, thereby overestimating its performance. 

\textbf{Results}: In~\cref{fig:result_res_dec_err} w.r.t to the decision-loss, our method consistently outperforms the baselines by a considerable margin across tasks. In~\cref{fig:result_rw}, w.r.t to the estimate-loss, our method outperforms all the baselines except for SPO+ on the SP problem. In~\cref{app:runtime}, we plot the wall-clock time/epoch for all the methods, and observe that our method is comparable to the baselines and scales gracefully as the dimension and the number of training examples increase. These results demonstrate the strong empirical performance of our method compared to other baselines. 



\vspace{-2ex}
\section{Discussion}
\label{sec:discussion}
We presented a reduction of CIO to convex feasibility, which enabled us to guarantee linear convergence to the solution without additional assumptions such as degeneracy or interpolation. We further reduced it to ERM on a smooth, convex loss that satisfies the PL condition. This enabled us to use first-order optimizers and demonstrate strong empirical performance on real-world tasks while being computationally efficient. For future work, we aim to address the following areas: (1) since solving the QP takes a substantial amount of time, we plan to incorporate techniques from~\citet{lavington2023target} to allow for multiple updates to the model for every solve of the QP, (2) we intend to extend our framework to accommodate unknown constraints, broadening its applicability (3) finally, we aim to experiment with general non-linear convex objectives.

\clearpage

\section*{Impact Statement}
This paper presents work whose goal is to advance the field of Machine Learning. There are many potential societal consequences of our work, none of which we feel must be specifically highlighted here.

\section*{Acknowledgements}
We would like to thank Michael Lu, Anh Dang, Reza Babanezhad, Vibhuti Dhingra, Tarannum Khan, Karan Desai, Ashwin Samudre and Hardik Chauhan for helpful feedback on the paper. This research was partially supported
by the Natural Sciences and Engineering Research Council of
Canada (NSERC) Discovery Grant RGPIN-2022-04816. Anant Raj is supported by the a Marie SklodowskaCurie Fellowship (project NN-OVEROPT 101030817). 

\bibliography{ref}
\bibliographystyle{icml2024}

\clearpage
\appendix
\onecolumn

\newcommand{\appendixTitle}{%
\vbox{
    \centering
	\hrule height 4pt
	\vskip 0.2in
	{\LARGE \bf Supplementary material}
	\vskip 0.2in
	\hrule height 1pt 
}}

\appendixTitle
 
\section*{Organization of the Appendix}
\begin{itemize}

   \item[\ref{app:definitions}] \hyperref[app:definitions]{Definitions}
   
   \item[\ref{app:proofs}] \hyperref[app:proofs]{Theoretical Results}


    \item[\ref{app:synthetic-experiments}] \hyperref[app:synthetic-experiments]{Synthetic Experimental Results}
      
   \item[\ref{app:experiments-details}] \hyperref[app:experiments-details]{Additional Real-world Experiment Details}
   
\end{itemize}

\section{Definitions}
\label{app:definitions}
If the function $f$ is differentiable and $L$-smooth, then for all $v$ and $w$, 
\begin{align}
f(v) & \leq f(w) + \inner{\nabla f(w)}{v - w} + \frac{L}{2} \normsq{v - w},
\tag{Smoothness}
\label{eq:individual-smoothness}
\end{align}

If $f$ is convex, then for all $v$ and $w$,
\begin{align}
f(v) & \geq f(w) + \inner{\nabla f(w)}{v - w},
\tag{Convexity}
\label{eq:convexity}
\end{align}

If $f$ is $\mu$ strongly-convex, then for all $v$ and $w$,
\begin{align}
f(v) & \geq f(w) + \inner{\nabla f(w)}{v - w} + \frac{\mu}{2} \normsq{v - w} 
\tag{Strong Convexity}
\label{eq:strong-convexity}
\end{align}
\section{Theoretical Results}
\label{app:proofs}

\subsection{Generalizing our method to other classes of optimization problem:}
\label{app:gen_kkt_inverse}

In this section, we relax our assumption on the class of problem from LP to non-linear convex optimization objectives and constraints. We specify the condition on the objective for our reduction to hold; there is no extra condition (apart from convexity) on optimization constraints for reduction to hold. 
Similar to ~\cref{sec:feasibility-reduction}, we consider a single data-point  $(z, x^*) \in \mathcal{D}$ with $\lvert\mathcal{D}\lvert=N$, where $z \in \mathbb{R}^{d}, x^* \in \mathbb{R}^{m}$ we aim to find a $c \in \mathbb{R}^{p} $ such that $\hat{x}(c) = x^*$. Consider a non-linear convex optimization problem defined as :
\begin{equation}   
\begin{aligned}
    \hat{x}(c) &= \, \underset{x}{\arg \min} \,  \phi(x,c)\\
    \text{subject to}\qquad 
    & G(x)=0 \\ 
    & H(x) \leq0
\label{eqn: gen_opt_problem}
\end{aligned}
\end{equation}

where $G(x), H(x)$ consist of $k, l$ convex constraints respectively and,  $\phi(x,c), G_i(x), H_i(x)$ are (non-linear) convex function with respect to parameter $x$. 


\textit{Assumptions:} For this reduction to hold, the condition we have is $\frac{\partial \phi(c, x)}{\partial x}|_{x=x^*}$ should be convex. \textit{Example:} For LPs, the condition $\frac{\partial <c, x> }{\partial x}|_{x=x^*}$ simplifies to $c$ which is convex in $c$.

Writing the KKT optimality conditions for ~\cref{eqn: gen_opt_problem}, we get: 
\begin{align*}
    \frac{\partial \phi(x,c)}{\partial x} + 
    \frac{\nu \cdot \partial{G(x)}}{\partial x} + \frac{\lambda \cdot \partial{H(x)}}{\partial x}  &\in 0 \tag{stationary condition}\\
    \lambda &\geq 0 \tag{dual feasibility}\\
    \lambda \cdot H(x) &= 0 \tag{complementary slackness}\\
    G(x) &= 0 \tag{primal feasibility}\\
    H(x) &\leq 0 \tag{primal feasibility}
    \label{eqn: gen_kkt_conditions}
\end{align*}
where $\lambda \in \mathbb{R}^{l}, \nu \in \mathbb{R}^{k}$  are referred to as the dual-variables. The equations, $G(x)= 0, H(x) \leq 0$ come from problem definition. The first term is derived by differentiating the Lagrangian. The term $\lambda \cdot H(x) = 0$ represents the complementary slackness condition.

Morever, after substituing the value of $x^*$ in ~\cref{eqn: gen_kkt_conditions}, we get:

\begin{equation}
\begin{aligned}
 {\frac{\partial \phi(x,c)}{\partial x}}\Big|_{x=x^*} + 
    \frac{\nu \cdot \partial{G(x)}}{\partial x}\Big|_{x=x^*} + \frac{\lambda \cdot \partial{H(x)}}{\partial x}\Big|_{x=x^*}   &\in 0 \\
    \lambda &\geq 0 \\
    \lambda \cdot H(x^*) &= 0 
    \label{eqn: gen_kkt_condition_init} 
\end{aligned}
\end{equation}

Now, for a given optimal decision $x^*$, $H(x^*), G(x^*)$ are inherently satisfied. Thus, we omit them from the ~\cref{eqn: gen_kkt_condition_init}.
The term $ \frac{\partial G(x)}{\partial x}|_{x=x^*} $ represents gradient of $G(x)$ taken w.r.t $x$ and evaluated at $x^*$.

From our assumption, ${\frac{\partial \phi(x,c)}{\partial x}}\big|_{x=x^*}$ is convex in $c$ and as $\lambda, \nu$ are multiplied with constants. Thus the~\cref{eqn: gen_kkt_condition_init} is convex in $c, \lambda, \nu$.

Therefore, the feasible set $C$ encompassing all the values of $c$ s.t. $\hat{x}(c) = x^*$ can be written as:

\begin{equation}
    C = \Bigl\{c\,\, |\, \exists\, \lambda, \nu \, \text{ s.t. } \, {\frac{\partial \phi(x,c)}{\partial x}}\Big|_{x=x^*} + 
    \frac{\nu \cdot \partial{G(x)}}{\partial x}\Big|_{x=x^*} + \frac{\lambda \cdot \partial{H(x)}}{\partial x}\Big|_{x=x^*}  \in 0, \; \lambda \geq 0, \; \lambda \cdot H(x^*) = 0  \Bigl\}
\end{equation}

Moreover, the projection of point $\hat{c}$ onto set $C$ can be attained by solving the below QP:
\begin{align}
    q(\hat{c}) &= \,  \underset{c}{\arg \min} \, \|c - \hat{c} \|^2 \\
   \text{subject to}\qquad
   &{\frac{\partial \phi(x,c)}{\partial x}}\Big|_{x=x^*} + 
    \frac{\nu \cdot \partial{G(x)}}{\partial x}\Big|_{x=x^*} + \frac{\lambda \cdot \partial{H(x)}}{\partial x}\Big|_{x=x^*}  \in 0 \\
    &\lambda \geq 0 \\
    &\lambda \cdot H(x^*) = 0 
\end{align}

Thus, our framework can be used for a non-linear convex class of functions where the first-order KKT conditions are both sufficient and necessary and $ \frac{\partial \phi(x,c)}{\partial x}|_{x=x^*}$ is convex in $c$. Moreover, we do not assume any additional condition on constraints $G, H$ apart from convexity.  

\subsection{Equivalence between margin in KKT formulation and ~\cite{sun2023maximum}}
\label{app:kkt_equiv}
In this section, we derive the equivalent margin for the KKT formulation as in ~\citet{sun2023maximum}. We further show that our margin formulation can be extended to the degenerate LPs. Additionally, our margin formulation does not require specific handling in the case of degenerate/non-degenerate cases.


First, we will derive the same margin formulation as in ~\cite{sun2023maximum} for non-degenerate LPs in terms of KKT conditions.

In the case of non-degenerate LP, we denote $B$ as the basis set defined as $B := \{ j \mid x^*_j > 0 \}$ and $M$ as the set of indices not in $B$, i.e. $M = [n] - B$. $A_B$ represents the columns corresponding to the indices in set $B$, is invertible by definition. The reduced cost is given as $c_M - c_B (A_B)^{-1}A_M  \geq 0$ from LP optimality conditions. And to add margin $\margin$ in the reduced cost optimality conditions, ~\cite{sun2023maximum} proposed the following modification: $c_M - c_B (A_B)^{-1}A_M  \geq \margin$.

Recall the the KKT conditions for LP:
\begin{align}
    \nu^T A -c+ \lambda &=0  \label{eqn: kkt_lp_stattionary}\\
    \lambda &\geq 0 \label{eqn: kkt_lp_dual}\\
    \lambda \cdot x^* &=0 \label{eqn: kkt_lp_complementary}
\end{align}

From KKT conditions, we have the condition $\nu^T A -c+ \lambda =0$. Separating it row-wise for index $B, M$ respectively for matrix $A$ , we get: 
\begin{align}
\label{eqn:eqn_a21}
    \nu^T A_B -c_B + \lambda_B &=0\\
    \nu^T A_M -c_M + \lambda_M &=0 \label{eqn: M-index}
\end{align}
where $A_B, A_M$ are the the corresponding columns of matrix $A$ defined by set $B$ and $M$ respectively. 

Now, from ~\cref{eqn: kkt_lp_complementary}, we have $\lambda \cdot x^* =0$. For non-degenerate LPs, we have $x_B >0$; this implies that $\lambda_B=0$. Substituting this value in ~\cref{eqn:eqn_a21}, we get:
\begin{align}
   \label{eqn:eqn_a22}
    \nu^T &= c_B(A_B)^{-1} \tag{substituting $\lambda_B=0$}\\
    \lambda_M &= c_M - \nu^T A_M   \tag{rearranging~\cref{eqn: M-index}}\\
    \lambda_M &= c_M -  c_B(A_B)^{-1}A_M \label{eqn: kkt_reduced_cost}
\end{align}

From ~\cref{eqn: kkt_lp_dual}, we have $\lambda\geq 0$ that implies $\lambda_M \geq 0$. Thus, in ~\cref{eqn: kkt_reduced_cost}, we retrieve the reduced cost optimality condition from KKT formulation. Therefore, the equivalent of $c_M -  c_B(A_B)^{-1}A_M \geq \margin$ to the KKT formulation would be to impose the same constraint on $\lambda_M$, i.e. $\lambda_M \geq \margin$.

\subsubsection{Extension to degenerate case:}
In the case of degenerate LPs, $\exists i \in B$ s.t. $x_i =0$. Moreover, the set $B$ is no longer unique. Our experiments found that a choice of $B$ affects the results in ~\cite{sun2023maximum}.

Here, let us define a separate set of basis $B_d, M_d$  as: $B_d := \{ i \mid x^*_i > 0 \}$, $M_d := \{ i \mid x^*_i = 0 \}$. Note that $B_d, M_d$ differ from the standard definition of basis in LPs. 

For the degenerate case, we propose the following modification for margin, $\lambda_{M_d} \geq \margin$. This can be written as: 
\begin{equation}
    \forall i \in [n]; \;\; \lambda_i \mathcal{I}\{x^*_i = 0\} \geq \margin
    \label{eqn: kkt_margin_deg}
\end{equation}
Note that margin modification in ~\cref{eqn: kkt_margin_deg} is exactly the same as in the non-degenerate case. Thus, unifying the margin formulation for the degenerate and non-degenerate LPs. Moreover, this also resolves the problem of determining the basis for degenerate LPs.





\subsection{Proof of ~\cref{thm:smooth-convex}}
\label{proof:smooth-convex}
\propsmooth*
To prove that the loss function $h(\theta)$ is a smooth, convex function for a linear model, we define $\zeta_i({\theta})$ such that, 
\begin{align}
h(\theta) = \frac{1}{N}\sum_{i=1}^N \zeta_i(\theta)  \text{ where } \zeta_i({\theta}) & := \frac{1}{2} \min_{q \in C_i} \|z_i \theta - q\|^2 = \frac{1}{2}\  d^2(z_i \theta, C_i) \,,
\label{eq:zeta-def}
\end{align}
where $d^2(z_i\theta, C_i)$ represents squared Euclidean distance of a point $z_i\theta$ from set $C_i$ and $C_i$ represents the set of feasible cost vectors (defined in~\cref{eq:c-set-def}) for corresponding solution $x_i^*$. 

We will prove that for an arbitrary $z$, $\zeta(\theta) :=\frac{1}{2} \min_{q \in C} \|z \theta - q\|^2$ is 1-smooth and convex. This will prove the desired statement since the mean of 1-smooth convex functions is 1-smooth and convex. 

For this, we define 
\begin{align}
\psi({\theta}) & := \min_{c \in C} \|z\theta - c\| = d(z\theta, C) \,,
\label{eq:psi-def}
\end{align}
where $d(z\theta, C)$ represents Euclidean distance of a point $z\theta$ from set $C$ and $C$ represents the set of feasible cost vectors (defined in~\cref{eq:c-set-def}) for corresponding solution $x^*$. We will first prove that $\psi({\theta})$ is $1$-Lipschitz and convex and use that to prove the smoothness and convexity of $\zeta(\theta)$. 

\begin{lemma}
For $\|z\| \leq 1$, loss  $\psi({\theta})$ is $1$-Lipschitz.
\label{lemma: 1-lipchitz}
\end{lemma}
\begin{proof}: 
We first prove that the projection of a point onto a closed convex set $C$ is non-expansive. Consider two arbitrary points $x, y \in \mathbb{R}^d$ for this. Now, for a point $p \in C$ s.t. $p$ is the projection of $y$ on $C$, we know that $d(y, C) = d(y, p)$ and $d(x, C) \leq d(x,p)$. 

\begin{align}
    d(x, C) & \leq d(x,p) \leq d(x,y) + d(y,p) \tag{Triangle inequality}    
    = d(x,y) + d(y, C) \\
    \implies d(x, C) &- d(y, C) \leq d(x,y)
    \label{eqn: x_c_x_y}
\end{align}
Similarly, now consider point $q \in C$ s.t. $q$ is the projection of $x$ on C, we know $d(x, C) = d(x,q)$ and $d(y, C) \leq d(y, q)$. For the same reason, 
\begin{align}
    d(y, C) &\leq d(y,q) \leq d(y,x) + d(x,C) = d(y,x) + d(x, C)\\ 
    \implies d(y, C) &- d(x, C) \leq d(x,y)
    \label{eqn: y_c_x_y}
\end{align}

As, $d(y, x) = d(x,y)$, from ~\cref{eqn: x_c_x_y}, ~\cref{eqn: y_c_x_y}, we get $|d(y, C) - d(x,C)| \leq d(x,y) = \|x-y\| $. Thus, the projection to a convex closed set $C$ is non-expansive.

Now consider two points $\theta_1, \theta_2$ for function $\psi$, 
\begin{align}
    \|\psi(\theta_1) - \psi(\theta_2) \| & = \| d(z\theta_1, C)- d(z\theta_2, C) \| \tag{By definition of $\psi$} \\
     &\leq d(z\theta_1, z\theta_2) \tag{From the above relation} \\
     &= \| z\theta_1- z\theta_2\|  \\
     &\leq \|z\|\|\theta_1- \theta_2 \|  \tag{Cauchy-Schwartz} \\ 
     &\leq 1\|\theta_1- \theta_2 \| \tag{Since $\norm{z} \leq 1$ by assumption}
\end{align}
\end{proof}

\begin{lemma}
Loss function $\psi(\theta)$ is convex.
\label{lemma: psi_convex}
\end{lemma}
\begin{proof}
Consider two parameter values $\theta_1, \theta_2 $ and let the projection of $z\theta_1$, $z\theta_2$ on $C$ be $c_1$, $c_2$ respectively. Additionally, consider $\theta_3$ to be the convex combination of $\theta_1, \theta_2$ i.e. $\theta_3 := \lambda \theta_1 + (1-\lambda) \theta_2$ for a arbitrary $\lambda \in (0,1)$. 
 
Now, we can write:
 \begin{align}
     \lambda \psi(\theta_1) + (1-\lambda) \psi(\theta_2) & =  \lambda \| z\theta_1-c_1 \| + (1-\lambda) \| z\theta_2-c_2 \|  \\
     &\geq  \| \lambda ( z\theta_1-c_1)  + (1-\lambda)(z\theta_2-c_2) \|  \tag{Triangle Inequality} \\
     &= \| \lambda \, z\theta_1 + (1-\lambda) \, z\theta_2 -  \lambda c_1 - (1-\lambda) c_2\| \\
     &= \| \lambda z\theta_1 + (1-\lambda) z\theta_2 -  c\| \tag{$c := \lambda c_1 + (1-\lambda) c_2$} \\
     &= \| z\theta_3 -c\| \tag{By definition of $\theta_3$} \\
     &= d(z\theta_3, c) \geq d(z\theta_3, C) \tag{Since $C$ is convex, $c \in C$} \\
     & = \psi(\theta_3) \tag{By definition of $\psi$} \\
    \implies \lambda \psi(\theta_1) + (1-\lambda) \psi(\theta_2) &\geq  \psi(\theta_3)\\
    \implies \lambda \psi(\theta_1) + (1-\lambda) \psi(\theta_2) &\geq  \psi(\lambda \theta_1 + (1-\lambda)\theta_2)   
 \end{align}
Thus, the function $\psi$ is convex from the definition of convexity. 
\end{proof}

\clearpage
\begin{lemma}
For $\|z\| \leq 1$, loss  $\zeta({\theta})$ is $1$-smooth.
\end{lemma}
\begin{proof}

Consider two parameter values $\theta_1, \theta_2$ and lets denote the projection of $z\theta_1$ and $z\theta_2$ on $C$ as $c_1$ and $c_2$ respectively. 
\begin{align}
    \left \|\nabla \zeta(\theta_1) - \nabla \zeta(\theta_2) \right \| &= \left \|\nabla \left(\frac{1}{2} \|z\theta_1 -c_1 \|^2 \right) - \nabla \left(\frac{1}{2} \|z\theta_2 -c_2 \|^2 \right)  \right \| \tag{By definition of $\zeta$} \\
    &= \|z^T (z\theta_1 -c_1)  - z^T(z\theta_2 -c_2) \|  \\
    &\leq \|z\| \|d(z\theta_1, C)- d(z\theta_2, C) \| \tag{Cauchy Schwarz} \\
    &\leq \|\psi(\theta_1) - \psi(\theta_2) \| \tag{Since $\norm{z} \leq 1$ by assumption and by definition of $\psi$} \\
    & \leq \|\theta_1 - \theta_2 \| \tag{\cref{lemma: 1-lipchitz}}
\end{align}
Thus, the function $\zeta(\theta)$ is $1$-smooth.
\end{proof}

\begin{lemma}
Loss  $\zeta({\theta})$ is convex.
\label{thm:h2_cnvx}
\end{lemma}
\begin{proof}
Consider a function $g(x) = \frac{1}{2} x^2$ and note that $\zeta(\theta) = g(\psi(\theta))$. From~\cref{lemma: psi_convex}, we know that $\psi$ is convex. $g(x)$ is non-decreasing for $x \in \{\mathbb{R}^+ \cup 0\}$ and $\psi(\theta)$ is always non-negative. The composition of two functions is convex if $g$ is non-decreasing and $\psi$ is convex~\citep{boyd2004convex}. Thus, the composite function $\zeta$ is convex. 
\end{proof}

\subsection{Proof for ~\cref{thm:conv_proof}}
\label{proof:conv_proof}
\propconverge*
\begin{proof}
    Loss $h(\theta)$ is defined as:
    \begin{equation}
         h(\theta) = \frac{1}{2N} \sum_{i=1}^N \min_{q_i \in C_i} \| f_{\theta}(z_i) - q_i\|^2
    \end{equation}
where $C_i$ represents the set of feasible cost vectors (defined in~\cref{eq:c-set-def}) for data-point $i$. We assume $f_{\theta}$ is a linear model for this proof for which the function $h(\theta)$ is convex.

In order to better interpret $h(\theta)$, consider a point $c_\theta = (c_1, c_2, \ldots, c_N) \in \R^{Nm}$ such that $c_i = f_\theta(z_i)$. Since $c_\theta \in \mathcal{F}$, $h(\theta) = \frac{d^{2}(c_\theta, \mathcal{C})}{N}$ where $d^{2}(w,\mathcal{W})$ is the squared Euclidean distance of point $w$ to the set $\mathcal{W}$. Hence, minimizing $h(\theta)$ is related to minimizing the distance between the sets $\mathcal{F}$ and $\mathcal{C}$. 

Consequently, our loss can be reformulated as:
\begin{equation}
    h(\theta) = \frac{1}{2N} d(c_\theta, \mathcal{C})^2
\end{equation}

Let us denote $\tilde{\theta} \in \arg\min h(\theta)$ and the predicted point as $\hat{c} := Z \, \tilde{\theta} \in \cF$. Therefore, $h(\tilde{\theta}) = \frac{1}{2N} d(\hat{c}, \mathcal{C})^2$. 

We can prove the proposition by contradiction. Assume, $\hat{c} := Z \, \tilde{\theta} \in \cF$  and is not the closest point to set $\mathcal{C}$. Conversely, assume the closest point to the set $\mathcal{C}$ in $\mathcal{F}$ is given by $Z\, \theta_p$. Now, the loss $h(\theta_p) =\frac{1}{2N} d(Z \theta_p, \,\mathcal{C})^2$ and since, $Z\theta_p$ is the closest point in the set $\mathcal{C}$ in $\mathcal{F}$, this means, that $\theta_p$ is also the $\arg \min$ of $h(\theta)$. As $\tilde{\theta}$ is also the $\arg \min$, this implies that $h(\tilde{\theta}) = h({\theta}_p)$, Thus, $d(Z\tilde{\theta}, \mathcal{C}) = d(Z\theta_p, \mathcal{C})$. This implies that minimizing $h(\theta)$ leads to convergence to a point $\hat{c} = Z \tilde{\theta}$, which is the closest distance to $\mathcal{C}$. In the case where an intersection exists, the closest distance to $\mathcal{C} =0$; therefore, it converges in $\mathcal{F}\cap \mathcal{C}$.


\end{proof}

\clearpage
\subsection{Proof of ~\cref{thm:pocs-pgd}}

\proppocssgd*

\label{proof:pocs-pgd}
\begin{proof}
    Consider iteration $t$, and the current value of model parameters at iteration $t$ is denoted by  $\theta_t$. Let us denote the updated parameter at time $t+1$ as $\theta_{t+1}$(POCS) and $\theta_{t+1}$(ERM) for POCS update and first order pre-conditioned update on $h(\theta)$ respectively.
    
    Let $q_t = \mathcal{P}_{\mathcal{C}}(Z\theta_t)$ denote the projection of $Z\theta_t$.

    In POCS, the $\theta_{t+1} = \arg \min_{\theta} \frac{1}{2N}\|Z\theta-q_t\|^2$. Solving it exactly gives the projection of $q_t$ onto the set $\mathcal{F}$, which can be written as: 
    \begin{equation}
       \mathcal{P}_{F}(q_t) = Z\theta_{t+1}  \; \text{s.t.} \, \theta_{t+1} = (Z^TZ)^{-1}Z^Tq_t
    \end{equation}
    Thus, $\theta_{t+1}\text{(POCS)} = (Z^TZ)^{-1}Z^Tq_t$

    Considering first order update for function $h(\theta)$ with step-size $\eta$ and pre-conditioner $\left[ \frac{Z^TZ}{N} \right]^{-1}$ can be written as: 
    \begin{equation}
        \theta_{t+1} = \theta_t -  \eta \, \left[ \frac{Z^TZ}{N} \right]^{-1} \nabla h(\theta_t) \tag{from definition of preconditioner update}
    \end{equation}
    
    where, $\nabla h(\theta_t) = \frac{1}{N}Z^T(Z\theta_t  - q_t)$.

    Now, putting the values back in preconditioned update, we get the following:
    \begin{align}
        \theta_{t+1}&= \theta_t -  \eta [Z^TZ]^{-1} ( Z^T(Z\theta_t  - q_t))\\
        &= \theta_t(1-\eta) + [Z^TZ]^{-1}  Z^T q_t \\
        &= (Z^TZ)^{-1}  Z^T q_t \tag{for $\eta=1$}
    \end{align}

Thus, $\theta_{t+1}\text{(ERM)} =(Z^TZ)^{-1}  Z^T q_t$. Therefore, we can see that iterates produced by POCS is equivalent to $1$-step of preconditioned gradient update with $\eta=1$.
\end{proof}

\subsection{Proof for ~\cref{thm:pl}}
\label{proof:pl}
\proppl*
We prove that $h(\theta)$ is not necessarily a strongly convex function by contradiction. Let us assume that $h(\theta)$ is $\alpha$-strongly convex function with $\alpha > 0$.  From the definition of $\alpha$-strong convexity, $h(\theta)$ must satisfy this following inequality for all $\theta_1, \theta_2$.

\begin{figure}[h]
    \centering
    \includegraphics[width=0.25\linewidth]{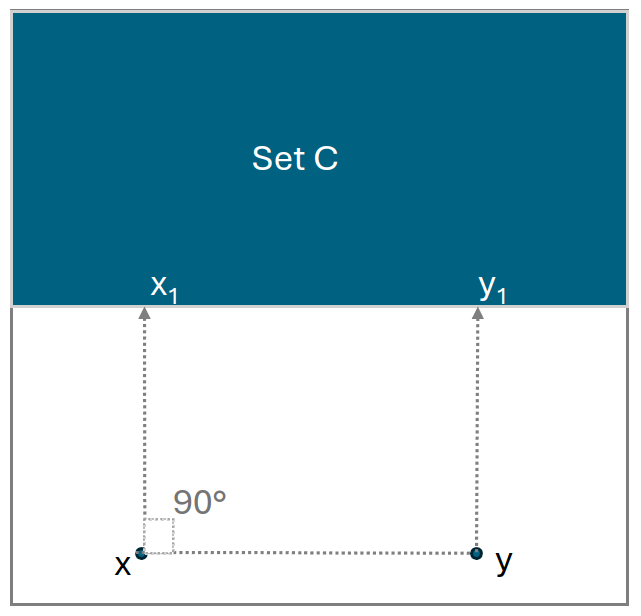}
    \caption{In this figure, we can see two point $x, y$ and their projection onto a linear boundary of set $C$ denoted as $x_1, y_1$ respectively. Moreover, the angle between $x,y$ and $x, x_1$ is the right angle; thus, the two vectors are orthogonal.}
    \label{fig:not_strongly_convex}
\end{figure}
\begin{equation}
    h(\theta_1) \geq h(\theta_2) +  (\nabla h(\theta_2))^T (\theta_1- \theta_2)   + \frac{\alpha}{2} \|  \theta_1- \theta_2 \|^2
    \label{eqn: strong_cvx_cndn}
\end{equation}

Consider a special case where $N = 1$ point and $m = d = 1$ and $z = 1$. Let $y = z\theta_1 = \theta_1$ and $x=z\theta_2 =  \theta_2$. Consider $C$ as an affine set and two points $(x, y)$ equidistant from $C$. Define $x_1, y_1$ to be the projection of $x$ and $y$ onto C respectively (refer to~\cref{fig:not_strongly_convex} above). 

Since $x$ and $y$ are equidistant from $C$, $\norm{x - x_1} = \norm{y - y_1}$. Moreover, since $x$ and $y$ are on the same side of $C$, vector $x - y$ is orthogonal vector $x - x_1$. Hence, $\langle y - x, x - x_1 \rangle = 0$. Substituting these values in ~\cref{eqn: strong_cvx_cndn}, we get: 
\begin{align}
    \frac{1}{2} \|y-y_1\|^2 &\geq  \frac{1}{2} \|x-x_1 \|^2 +\langle x - x_1, y - x \rangle + \frac{\alpha}{2} \|x-y\|^2  \\    
     0 &\geq \frac{\alpha}{2} \|x-y\|^2 \\
     \implies \alpha &=0 
\end{align}
As $\alpha =0$ therefore the function $h(\theta)$ is not strongly convex for this case when $C$ is an affine set. However, we can show that function $h(\theta)$ satisfies the PL inequality with $\mu = \frac{\lambda_{\min} [Z^TZ]}{N}$. Recall,  $h(\theta)$ is defined as $\frac{1}{2N} \|Z\theta - q\|^2$ where $q = \mathcal{P}_{\mathcal{C}}(Z\theta)$ is projection of $Z\theta$ on $\mathcal{C}$. Hence, $\nabla h(\theta) = \frac{1}{N}Z^T [Z\theta - q]$.

\begin{align}
\|\nabla h(\theta)\|^2 &= \frac{1}{N^2} \normsq{Z^T [Z\theta - q]}  \tag{By definition of $\nabla h(\theta)$}\\
& \geq \frac{1}{N^2}\sigma^2_{\min}(Z^T)  \normsq{[Z\theta - q]} \tag{$\norm{Ax} \geq \sigma_{\min}(A) \, \norm{x}$}\\
& =  \frac{1}{N^2}\lambda_{\min} (Z^TZ)  \normsq{[Z\theta - q]} \tag{using $\sigma^2_{\min}(Z^T) = \lambda_{\min} (Z^TZ)$}\\
& = \frac{2}{N}\lambda_{\min} (Z^TZ) \, h(\theta) \tag{By definition of $h$} \\
& \geq \frac{2}{N}\lambda_{\min} (Z^TZ) [h(\theta) - h(\theta^*)] \tag{Since $h$ is non-negative}\\ 
& = \frac{2}{N}\lambda_{\min} \left(\sum_{i = 1}^{N} z_i z_i^\top \right) [h(\theta) - h(\theta^*)] \tag{replacing $Z^TZ$ as $\sum_{i = 1}^{N} z_i z_i^\top$}\\ 
& = 2 \mu \, [h(\theta) - h(\theta^*)] \tag{For $\mu = \lambda_{\min}\left[\frac{\sum_{i = 1}^{N} z_i z_i^\top}{N}\right]$}
\end{align}
Hence, $h(\theta)$ is $\mu$-PL.

\clearpage
\subsection{Sub-optimality proofs}

\propsubopt*
\begin{proof}
\label{proof:subopt}
\begin{align}
\Gamma(\theta, (z, x^*)) &= \left \langle \frac{\mathcal{P}_{C} (c_{\theta})}{\norm{\mathcal{P}_{C} (c_{\theta})}}, \, \hat{x}(c_{\theta}) - x^* \right \rangle \tag{By definition}\\
    & \leq \left \langle \frac{\mathcal{P}_{C} (c_{\theta})}{\delta}, \, \hat{x}(c_{\theta}) - x^* \right \rangle \tag{Since by~\cref{thm:c_hat_min}, $\norm{\mathcal{P}_{C} (c_{\theta})} \geq \delta$} \\
        &\leq \langle \frac{ \mathcal{P}_{C}(c_{\theta})}{\delta} ,\,   \hat{x}(c_{\theta})  -x^*  \rangle + \frac{1}{\delta}\langle c_{\theta},\, x^*  - \hat{x}(c_{\theta}) \rangle \tag{By definition of $\hat{x}(c_\theta)$, $\langle c_{\theta},\, x^* \rangle \geq \langle c_\theta, \hat{x}(c_{\theta})  \rangle$}\\
           &= \frac{1}{\delta}\langle \mathcal{P}_{C} (c_{\theta})  - c_{\theta} , \,  \hat{x}(c_{\theta})  -x^*  \rangle \tag{rearranging terms}\\
           &\leq \frac{1}{\delta}\|\mathcal{P}_{C} (c_{\theta}) - c_{\theta}  \| \| \hat{x}(c_{\theta})  -x^*  \| \tag{Cauchy-Schwartz}\\
           &\leq \frac{1}{\delta}\|\mathcal{P}_{C} (c_{\theta}) - c_{\theta}  \| \sqrt{m} \tag{From assumption that $\forall j, \hat{x}_j, x_j^* \in [0,1]$}\\
           & = \frac{1}{\delta}\sqrt{2h(\theta)} \sqrt{m} \tag{from definition of $h(\theta)$} \\
           &= \frac{\sqrt{2m \, h(\theta)}}{\delta}
\end{align}
\end{proof}

Next, we give the lower-bound on the term $\delta$ which depends on the margin $\margin$ defined in~\cref{sec:practical}.
\begin{restatable}{lemma}{lemmaminc}
\label{thm:c_hat_min}
For $c \in C$ defined using a margin $\margin$, $\norm{c}_{2} \geq \delta := \frac{\margin}{\sqrt{m}} \, \max_{j \in M} \min \left \{1, \min_{p \in B} \frac{1}{|\tau_{pj}|} \right \}$. 
\end{restatable}
\begin{proof}
\label{proof:minc}
We lower-bound $\norm{c}_{2}$ for an arbitrary $c \in C$ using the reduced cost optimality conditions~\citep{luenberger1984linear}. For this, we denote $B$ as the basis set defined as $B := \{ j \mid x^*_j > 0 \}$ and $M$ as the set of indices not in $B$, i.e. $M = [m] - B$. Let us define a new term $\tau_{ij} = [(A_B)^{\dagger}A_j]_i$ where $A_B$ represents the columns corresponding to the indices in set $B$, and $A_B^{\dagger}$ is the pseudo-inverse of the matrix $A_B$. 

To prove this proposition, we first consider two arbitrary vectors $a, b \in \mathbb{R}^{m}$ and show that $\|a \|_1 \geq \frac{\margin}{\max_{p=1}^n|b_p|}$ is a \emph{necessary condition} to ensure that $a^T b \geq \margin$. We do this by contradiction: assume $\|a \|_1 \leq \frac{\margin}{\max_{p=1}^m|b_p|}$, but $a^T b \geq \margin$. In this case, 
\begin{align*}
\|a \|_1 \, \max_{p=1}^m|b_p| & \leq \margin \implies \|a\|_{1} \, \|b\|_{\infty} \leq \margin \\
\intertext{By Holders inequality, since $a^T b \leq \|a\|_{1} \, \|b\|_{\infty}$, the above inequality implies that}
a^T b \leq \margin \,,
\end{align*}
which is a contradiction. Since $\|a \|_1 \geq \frac{\margin}{\max_{p=1}^m|b_p|}$, it gives us a lower-bound on $\|a\|_{1}$. Now, we can use this result to find the lower bound on $\|c\|_{1}$. 

In~\cref{app:kkt_equiv}, we have shown that KKT conditions are equivalent to reduced cost optimality conditions. We first find the lower bound to satisfy reduced cost inequality for a specific index $j \in [M]$ and then extend the result for all indices in $M$ to find the lower-bound on $\|c\|_1$.

The reduced cost for an index $j \in M$ and margin $\margin$ is given by $r(j) := c_j - c_B (A_B)^{\dagger} A_j$. The reduced costs conditions imply that for all $j \in M$, $r(j) \geq \margin$. In terms of $\tau$, these conditions imply that $c_j - \sum_{p \in B} c_p\tau_{pj} \geq \margin$. Equivalently, for all $j \in M$, $c^T\alpha_j \geq \margin$, where $\alpha_j$ represents the coefficients of $c$ in $r(j)$. 

Using the above result to obtain a lower-bound on $\|c\|_{1}$, we have that, 
\begin{align}
    \|c\|_1 &\geq \frac{\margin}{\max_{p=1}^m |(\alpha_j)_p|} \\  
    &= \frac{\margin}{\max\{1, \max_{p \in B} |\tau_{pj}|\}} \tag{substituting the value of $\alpha_j$} \\
    &= \margin \, \min \left \{1, \min_{p\in B} \frac{1}{|\tau_{pj}|} \right \} \tag{rearranging terms}
\end{align}

Since we require that the reduced cost condition be satisfied for all $j \in M$, we get that, 
\begin{align*}
\|c\|_{1} \geq \margin \, \max_{j \in M} \min \left \{1, \min_{p\in B} \frac{1}{|\tau_{pj}|} \right \}
\end{align*}

Finally, we use the relation between norms to lower-bound the value of $\|c\|_{2}$.
\begin{align}
    \delta &:= \min_{c \in C} \|c\|_2  \\
           &\geq \min_{c \in C} \frac{1}{\sqrt{m}} \|c\|_1     \tag{by norm inequality}\\           
           &\geq \frac{\margin}{\sqrt{m}} \, \max_{j \in M} \min \left \{1, \min_{p \in B} \frac{1}{|\tau_{pj}|} \right \}
\end{align}
\end{proof}

\section{Synthetic Experimental Results}
\label{app:synthetic-experiments}


We conduct numerical experiments for two  LP problems – the shortest path (SP)  problem and the fractional Knapsack problems considered in ~\cite {sun2023maximum}. For both SP and Knapsack, we generate $100$ samples for training, validation and test sets. We used the codebase provided by the \cite{sun2023maximum} to generate the dataset.

\textbf{Shortest Path (SP-synth)}: The Shortest Path problem is defined on a $5 \times 5$ grid with $m=40$ directed edges associated with the ground truth cost-vector $c^* \in \mathbb{R}^m$. Input $z \in \mathbb{R}^d$ with $d=6$. Thus, $\theta \in \mathbb{R}^{d \times m}$. To make the problem harder, we use the degree$=4$ in the data-generation process. 

\textbf{Fractional Knapsack}: The Fractional Knapsack problem is defined with input $z \in \mathbb{R}^d$ with $d=5$. We have $10$ items with associated cost-vectors, and slack variables are added to convert the problem to standard form, making the dimension $m=21$. Thus, $\theta \in \mathbb{R}^{d \times m}$. To make the problem harder, we use the degree$=2$ in the data-generation process with the attacking noise of attack-power$=3.0$. 

\textbf{Methods and model}: For the experiments, we compare both our variants, POCS (\cref{alg:revgrad}) and ERM (\cref{sec:erm}) with GD and show that they have similar performance. We compare our method against against several existing methods, including ST~\cite{sahoo2022backpropagation}, MOM \cite{sun2023maximum}, BB \cite{vlastelica2019differentiation},  QPTL ~\citet{wilder2019melding} and SPO+ \cite{elmachtoub2022smart}.
We train all the models in a deterministic setting employing a linear model. 

\textbf{Training Details}
For Ours (POCS), we solve the regression problem using closed-form solutions obtained through matrix inversion. For the MOM, Ours (ERM) method, we employ the Armijo line search algorithm ~\citep{armijo1966} with Gradient Descent.  For the  BB, ST, QPTL, and SPO+ methods, a grid search is used to find the best constant step-size in $\{10, 1, 0.1, 0.01, 0.001, 0.0001\}$.  We also did a grid search for $\lambda$,  regularizer in QPTL and the perturbation weight in BB with $\lambda \in \{100, 10, 1, 0.1, 0.01, 0.001 \}$. The optimal settings are determined based on performance on the validation set, and we plot the training and test plot for the best-performing model.

\begin{figure}[htbp]
\vspace{2ex}
    \centering
    \begin{minipage}{0.48\textwidth}
\includegraphics[width=\linewidth]{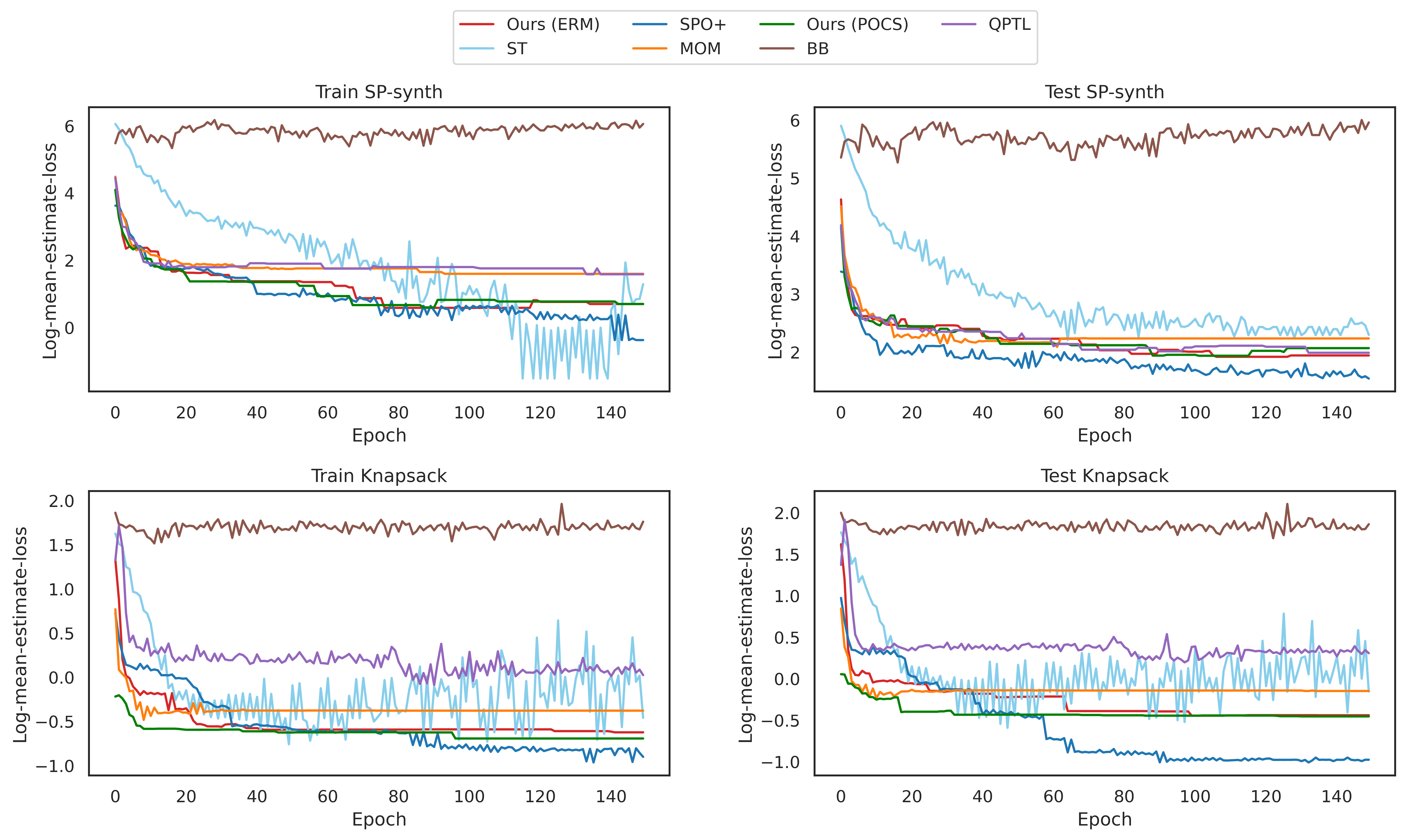}
\caption{Estimate loss}
\label{fig:result_gen}
 \end{minipage}\hfill
    \begin{minipage}{0.48\textwidth}
    \centering
\includegraphics[width=\linewidth]{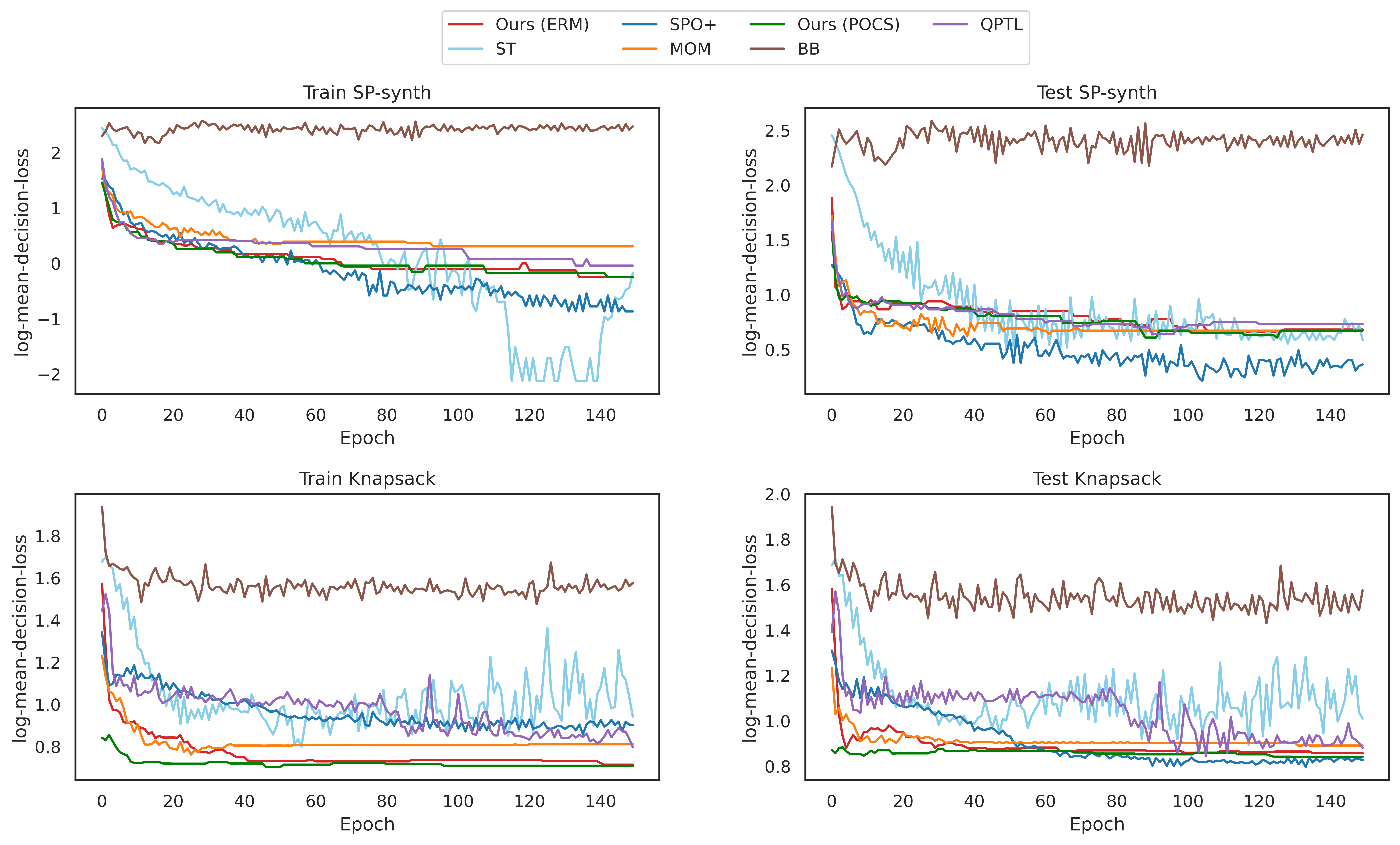}
\caption{Decision loss}
\label{fig:result_gen_dec}
\end{minipage}
    \caption{: Training and Test plot for synthetic tasks. For both problems, our method significantly outperforms the other methods (ST, QPTL, BB, MOM) and is comparable to SPO+, which uses the knowledge of $c^*$}
    \label{fig:side_by_side_minipage}
\end{figure}

\textbf{Results}:
In~\cref{fig:result_gen_dec} w.r.t to the decision-loss, all the methods have similar performance except SPO+, BB. For Knapsack, our method outperforms all the other baselines except SPO+.
In~\cref{fig:result_gen}, w.r.t to the estimate-loss,  our method significantly outperforms the other methods (ST, QPTL, BB, MOM) and is comparable to SPO+, which uses the knowledge of $c^*$. Moreover, we can see that both our variants, POCS and ERM, have negligible performance differences.


\section{Additional real-world experiment details}
\label{app:experiments-details}

\subsection{Warcraft shortest Path}
\label{app: sp_wc}

In this section, we define the LP for the shortest path problem. Consider $ x_{ij} $ represents the edge from vertex $ i $ to vertex $ j $. Since all edges are bidirectional, $ x_{ij} $ is not the same as $ x_{ji} $.  $ s $ and $ t $ denote source and target vertices, respectively. Let $ c_{ij} $ be the cost of selecting edge $ x_{ij}$. Before initializing the LP, let $ N(i) $ be the set of vertices with an outgoing edge from vertex $ i $, and $ I(i) $ be the set of vertices with an incoming edge to vertex $ i $. The LP can be written as follows:

\begin{align*}
\underset{x}{\text{minimize}} \qquad & c^T x \\
\text{subject to} \qquad
\forall i \notin \{s, t\}, \; \sum_{j \in N(i)} x_{ij} & = \sum_{j \in I(i)} x_{ji} \quad \text{(flow conservation)}\\
\sum_j x_{sj} & = 1 \quad \text{(source has one outgoing edge)} \\
\sum_j x_{jt} & = 1 \quad \text{(target has one incoming edge)} \\
x & \geq 0 
\label{eq: sp_lp}
\end{align*}

In this case, the source is always in the top-left grid, and the target is in the bottom-right grid. As in the dataset, ground-truth weights are defined on the vertex. To run it as LP defined in~\cref{eq: sp_lp}, we directly predict the weights of the edges. Given that the ground-truth cost $ c^*(BB) $ is defined for vertex weights in the dataset, we define $ c^* $ for the edge-weighted shortest path as:
\begin{align}
\forall i, \; \forall j \in N(i), \; c^*_{ij} = c^*_i (BB)
\end{align}
Both approaches yield the same shortest path, validating the conversion. Please see \cite{vlastelica2019differentiation} for more details regarding the dataset.

\begin{figure}
    \centering
    \includegraphics[width=1\linewidth]{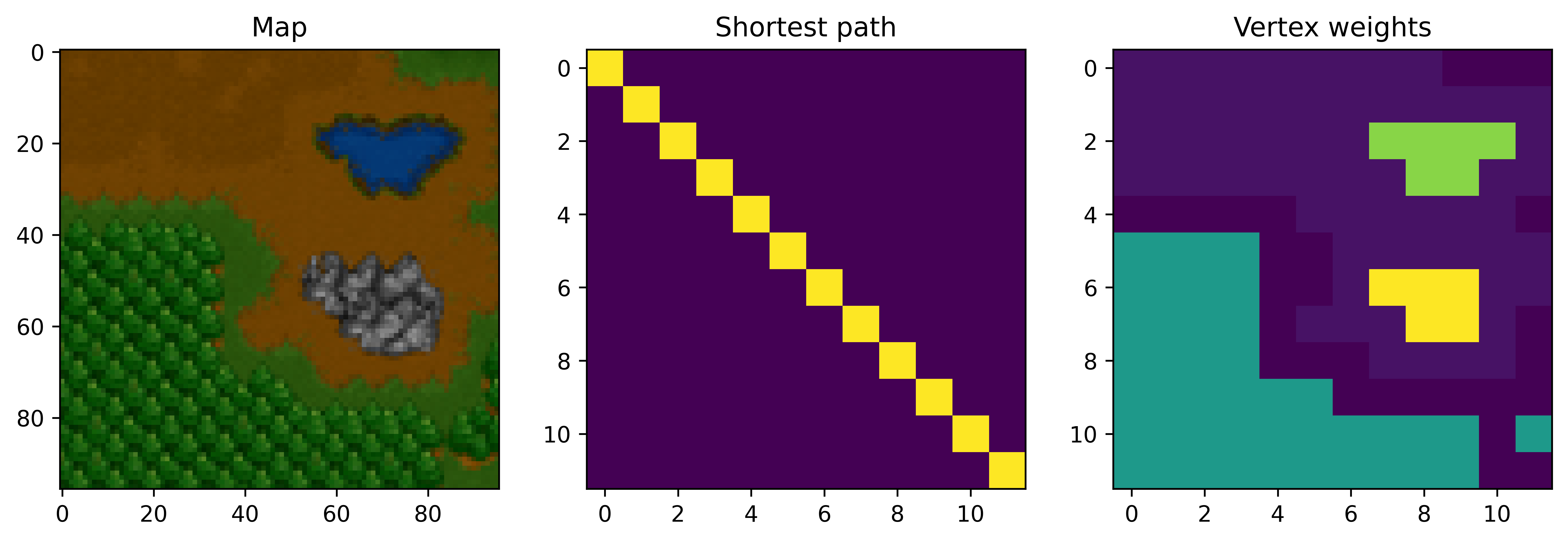}
    \caption{Warcraft SP dataset sample: The input image (left), ground truth shortest path (center), and the ground truth vertex weights (right). The task is to learn the edge weights to retrieve the same shortest path.}
    \label{fig:warcraft-example}
    \vspace{4ex}
\end{figure}

\subsection{Perfect Matching}

 
\label{app: pm_lp}
 In this section, we define the LP associated with the perfect matching problem. $x_{ij}$ represents the edge from vertex $i$ to vertex $j$. $N(i)$ represents the neighbors of vertex $i$. All the edges in the graph are unidirectional, i.e. $x_{ij}$ and $x_{ji}$ represent the same edge. $c_{ij}$ represent the cost of selecting the edge $x_{ij}$. Thus, the LP can be written as:

 \begin{align*}
\underset{c}{\text{minimize}}\qquad & c^Tx  \\
\text{subject to}\qquad 
    & \forall  i  \; \sum_{j \in N(i)} x_{ij} =1 \tag{each vertex should have exactly one incident edge}\\ 
    & x \geq 0 
\end{align*}
\begin{figure}[h]
    \centering
    \includegraphics[width=0.75\linewidth]{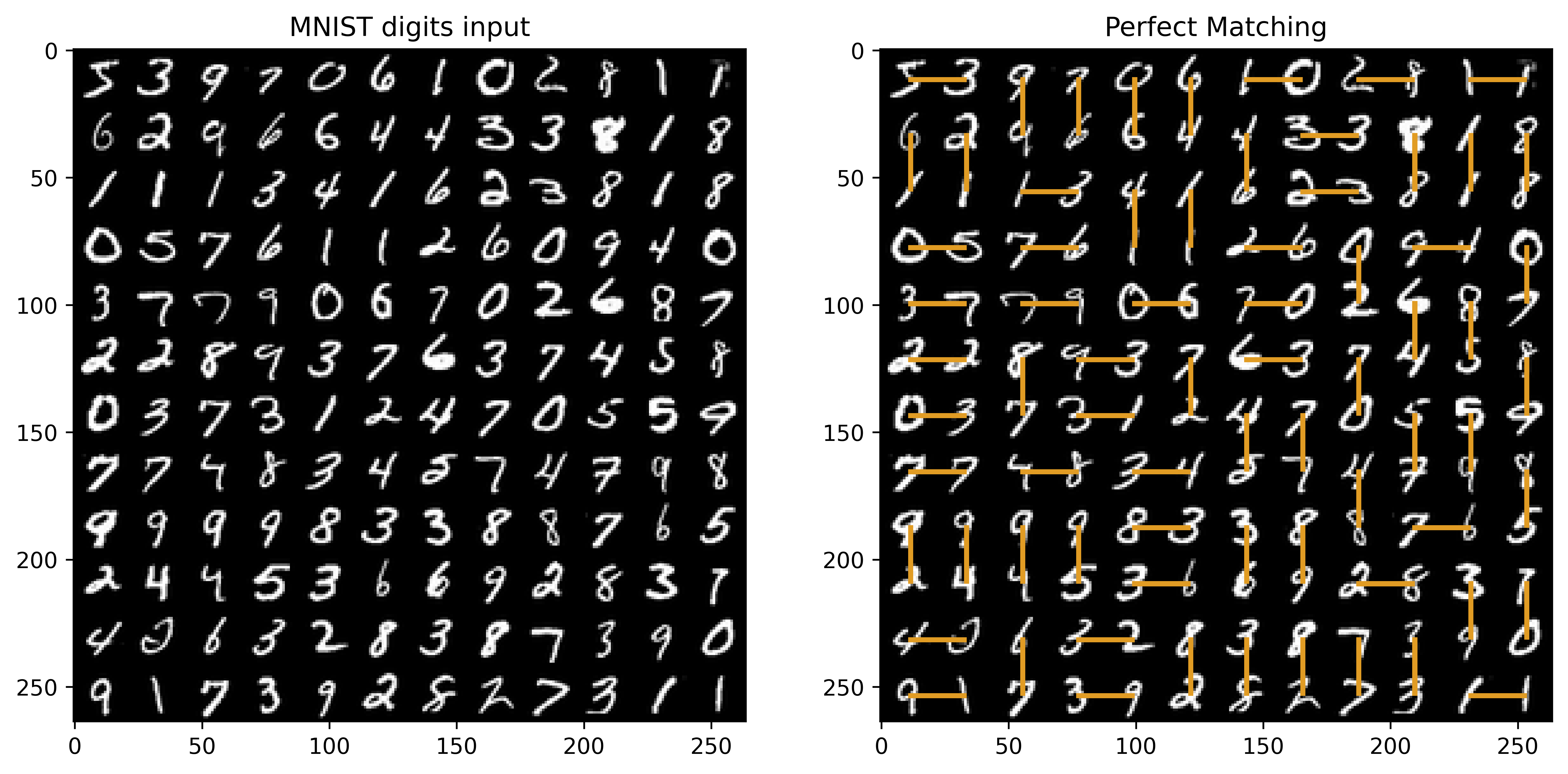}
    \caption{Perfect Matching dataset sample: This figure shows the input image (left) and the corresponding min-cost perfect matching overlayed on the input image on the right. Each input is a $12 \times 12$ grid, with each grid containing an MNIST digit. In Perfect Matching(PM), edges highlighted by the orange lines represent the edge selected by the solving min-cost perfect matching optimization problem. Ground truth edge weights are inferred by reading the digits connected by the edge as a two-digit number. The task is to predict edge weights such that we get the same PM.}
    \label{fig:pm-dataset}
\end{figure}

In our case, ground-truth edge weights are inferred by reading the digits on the two ends of the vertex as two-digit numbers. Please see \cite{vlastelica2019differentiation} for more details regarding the dataset.
\vspace{1ex}
\subsection{Runtime Comparison}
\label{app:runtime}
\begin{figure}[!h]
\begin{center}
\centerline{\includegraphics[scale=0.4]{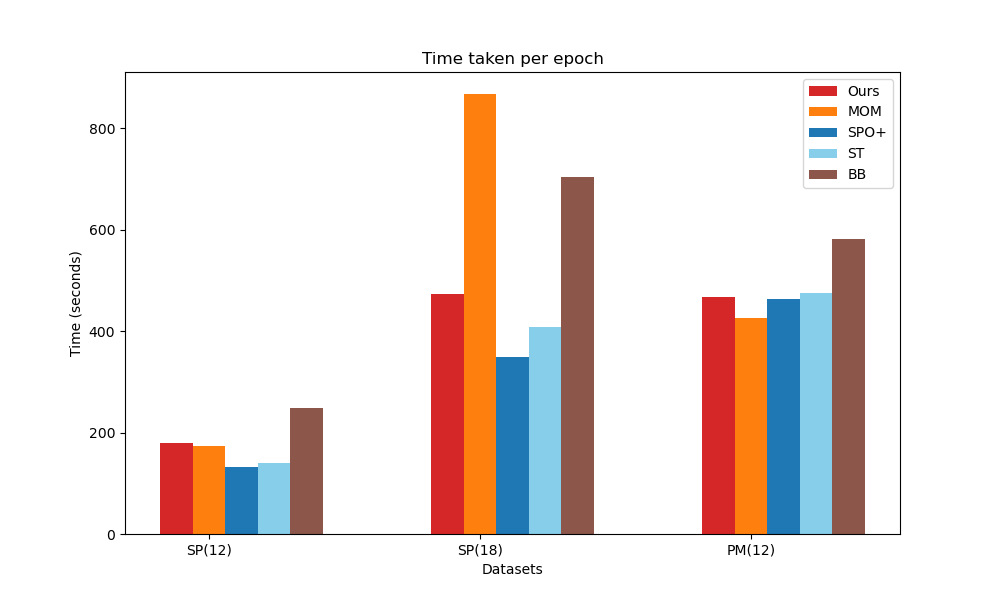}}
\caption{Training Time (in seconds) per epoch vs method for three real-world experiments. We can see that our method is comparable to other methods and  scales well with the dimension of the problem}
\label{fig:bar_plot}
\end{center}
\end{figure}

To benchmark the computational efficiency of our method, we plot the average training time per epoch for all methods in~\cref{fig:bar_plot}. Our method is competitive with ST and SPO+, and faster than MOM and BB. BB requires two solver calls per gradient evaluation, while MOM, despite not needing solver calls, involves inverting a matrix $ A_B $, which scales poorly with dimension. As shown in the plot, our method scales well with both the problem dimension and dataset size.


\subsection{Ablation study for margin}
\label{app:margin}
In order to verify the robustness of~\cref{alg:revgrad} for varying $\margin$, we do an ablation study varying $\margin$ from [10, 0.01] for the synthetic datasets in the deterministic setting (refer to ~\cref{app:synthetic-experiments} for details). We train each model for $150$ epochs according to \cref{alg:revgrad} and report the mean decision error (\cref{eqn:decision-loss}) for both the train/test data for the synthetic shortest-path (SP) and Knapsack (K) problems. We see that increasing the margin (from $0.01$ to $1$) leads to small sub-optimality and that increasing it beyond $1$ does not improve the performance.

\begin{table*}[!h]
\centering
\begin{tabular}{@{}rcccc@{}}
\toprule
Margin & Train Sp-synth & Test Sp-synth &  Train Knapsack & Test Knapsack\\ 
\midrule
  10 & 0.779 & 1.95 & 2.033 & 2.32 \\ 
 1 & 0.779 & 1.89 & 2.033 & 2.32 \\ 
 0.1 & 0.699 & 2.11 & 2.08 & 2.39 \\ 
 0.01 & 2.63 & 3.23 & 2.11 & 2.32 \\ 
\bottomrule
\end{tabular}
\caption{Ablation results varing margin ($\margin$) from $10$ to $0.01$ and their performance on train/test set for the synthetic dataset }
\label{tab:margin-ablation}
\end{table*}

\end{document}